\documentclass{article} 
\usepackage{iclr2024_conference,times}
\usepackage{graphicx}

\usepackage{amsmath,amsfonts,bm}

\def\eqref#1{equation~\ref{#1}}

\def\1{\bm{1}}

\DeclareMathAlphabet{\mathsfit}{\encodingdefault}{\sfdefault}{m}{sl}
\SetMathAlphabet{\mathsfit}{bold}{\encodingdefault}{\sfdefault}{bx}{n}

\usepackage{hyperref}
\usepackage{url}
\usepackage{amsthm}
\newtheorem{theorem}{Theorem}
\newtheorem{definition}[theorem]{Definition}
\newtheorem{lemma}[theorem]{Lemma}
\newtheorem{corollary}{Corollary}[theorem]
\newtheorem{proposition}{Proposition}[theorem]
\title{Flat Minima in Linear Estimation and an Extended Gauss Markov Theorem}

\author{Simon Segert\\
Princeton Neuroscience Institute\\
Princeton University\\
Princeton, NJ \\
\texttt{ssegert@princeton.edu} \\
}
\iclrfinalcopy 

\begin{document}

\maketitle

\begin{abstract}
We consider the problem of linear estimation, and establish an extension of the Gauss-Markov theorem, in which the bias operator is allowed to be non-zero but bounded with respect to a matrix norm of Schatten type.  We derive simple and explicit formulas for the optimal estimator in the cases of Nuclear and  Spectral norms (with the Frobenius case recovering ridge regression). Additionally, we analytically derive the generalization error in multiple random matrix ensembles, and compare with Ridge regression. Finally, we conduct an extensive simulation study, in which we show that the cross-validated Nuclear and Spectral regressors can outperform Ridge in several circumstances.
\end{abstract}

\section{Introduction}
Linear models are among the most used of all machine-learning models in applications to science and engineering. In addition to this practical interest, they are also of great theoretical interest. Indeed, the empirical successes of neural networks have proven somewhat at odds with traditional received wisdom from classical statistical learning theory\citep{belkin2}. This has begun to be reconciled by careful analysis of well-chosen linear ``model systems" such as high-dimensional regression \citep{hastie}, kernel ridge regression \citep{canatar,jacot}, or linear neural networks \citep{saxe}, which can reproduce many qualitative properties of learning dynamics of neural networks. Thus, the careful study of linear models, especially in the limit of a large number of features and observations can prove highly valuable for qualitative understanding of non-linear models. 

Additionally, it has been found that explicit regularization is often a key ingredient for achieving high performance of both linear and non-linear models. For example, neural networks are often trained using L2 regularization \citep{goodfellow} or dropout \citep{srivastava}. By contrast, most of the above-mentioned theoretical work on linear models considers specifically L2 regularization (i.e. Ridge regression). It is quite natural, then, to think that the linear setting could be used as a test case for studying other forms of regularization. This question becomes especially interesting in light of several recent works which have explored the effect of various non-standard forms of regularization on neural networks such as rank constraints \citep{rao,yang} or Spectral norms \citep{johansson,yoshida}. The Nuclear norm (or convexified rank) also plays a key role in other kinds of high-dimensional non-linear learning problems \citep{hu}, most notably matrix completion, and is related to dropout regularization in neural networks \citep{mianjy}.

Thus with the general motivation of gaining a detailed understanding of different kinds of regularizations in a tractable yet previously-validated and informative setting, we study in this work the effect of different matrix norm regularizations in the context of high-dimensional linear regression. More specifically, we consider the problem of linear regression, under the constraint that the \textit{bias matrix} is bounded according to some fixed matrix norm. We prove a Gauss-Markov-like theorem for this setting, which characterizes the optimal estimator at a fixed level of bias,  assuming the norm in question is of Schatten type.  We will be specifically interested in Nuclear, Frobenius, and Spectral norms, in which case the optimal estimators have especially simple forms, which we derive.

Next, we characterize the test error of Nuclear and Spectral estimators in several random matrix ensembles (the corresponding results for Ridge regression being well known). We find an intriguing pattern: while the Ridge estimator can always attain the lowest error of any estimator, the minima can be  \textit{sharp} (as a function of regularization strength). The Nuclear estimator, by contrast, can often attain test error nearly as low as the Ridge case, but with significantly flatter minima. This has implications for practical situations when the regularization strength is selected using a noisy search procedure such as cross validation, as it could happen that such a procedure finds a better solution in the latter case, even if the actual global optimum is better in the former case. We then perform a series of simulation studies with cross-validation, in which we find that this can actually happen, with both Gaussian features and Random Fourier features, and more generally characterize the relative performance of the three estimators across a range of several generative factors. We close with a survey of related work and general discussion.
\section{Theory}
\subsection{Setup and Extended Gauss-Markov Theorem for Schatten norms}
\label{main_theorem_section}
Throughout, we will be concerned with the classic regression problem. Denote the training data by $X$, which is an $N\times d$ matrix, and the training targets by $Y$, which is an $N\times 1$ vector. $Y$ is assumed to have distribution $Y=X\beta_0+\epsilon$, where $\epsilon$ are independent samples from a fixed noise distribution. 

We will assume that the noise distribution has mean zero and variance $\sigma^2=1$ , but otherwise do not place distributional assumptions on it. We will restrict attention to the class of \textit{Linear Estimators}. These are models that take the form
\begin{equation}
\label{linear_estimator_defn}
\hat{\beta}=LY
\end{equation}
Where $L$ is an $d\times N$ matrix that depends (possibly non-linearly) on the training data. Given a new datapoint $x_{test}$,the predicted value is just $\hat{y}=\langle x_{test},\hat{\beta}\rangle$. One simple observation is that for any linear estimator the bias $\mathbb{E}_{\epsilon}\hat{\beta}-\beta_0$ can be expressed as a linear function of $\beta_0$:
\begin{equation}\mathbb{E}_{\epsilon}\hat{\beta}-\beta_0=(LX-I)\beta_0\end{equation}
We thus define the bias operator $B:=LX-I$.  Note that the actual bias value generally depends on the unknown quantity $\beta_0$, while the variance $Var_{\epsilon}(\hat{\beta})=LL^T$ does not. The classic Gauss-Markov theorem tells us that if we want to impose exactly zero bias\footnote{Note the subtle point that this is the only case in which we can exactly control the bias without knowing the value of $\beta_0$}, then the minimal variance estimator is the ordinary least squares estimator. But what if we want to relax this, to allow non-zero but bounded amount of bias? Taking this desideratum literally, we run into the fundamental issue that the bias depends on the unknown $\beta_0$. As a proxy, we propose to control the size of $B$ instead, using some choice of matrix norm. In what follows, we will only consider Schatten norms, $\|M\|_p=(\sum_i \sigma_i(M)^p)^{1/p}$, $p\geq 1$, although in principle other matrix norms could be used. \footnote{We reserve the notation $\|\cdot\|_p$ for a Schatten norm of a matrix, and use $|\cdot|_p$ for Euclidean vector norms}This motivates the following:

\begin{definition}
Let $p\geq 1$, and let $C\geq 0$. For any matrix $X\in\mathbb{R}^{N\times d}$,the \textit{p-Bias constrained Linear estimator} $L_p(X)$ is defined as the minimal variance estimator which has p-bias of at most $C$. That is, 
\begin{equation}
\label{main_objective}
L_p(X)=argmin_{L\in\mathbb{R}^{d\times N}; \|LX-I_d\|_p\leq C} Tr(LL^T)/2\end{equation}
where $\|\cdot\|_p$ is a Schatten norm.  If $Y$ is a vector of regression targets, the estimated coefficient vector is $\hat{\beta}:=L_p(X)Y$.
\end{definition}

We now state: 

\begin{theorem}
(Extended Gauss-Markov)
\label{main_theorem}
Let $X\in \mathbb{R}^{N\times d}$ be a matrix of observations. Assume that $G:=X^TX$ is invertible, and let $G=Udiag(\sigma) U^T$ be the diagonalized form. For $p\geq 1$, the p-Bias constrained estimator with bound $C$ can be expressed in the form $L_p(X)=\hat{G}^{-1}X^T$, where $\hat{G}$ is symmetric,simultaneously diagonalizable with $G$, and satisfies $G
\preceq\hat{G}$\footnote{i.e., $\hat{G}-G$ is non-negative definite}. For the following special cases of $p$ we further have:

$p=1$ (Nuclear norm): $\hat{G}=Udiag(\max(\sigma,\alpha))U^T$

$p=2$ (Frobenius norm): $\hat{G}=G+\alpha I_d$

$p=\infty$ (Spectral norm): $\hat{G}=(1+\alpha)G$

where $\alpha\geq 0$ is determined by $C$, and $max$ denotes the elementwise maximum, i.e. $max(\sigma,\alpha)_i=max(\sigma_i,\alpha)$
\end{theorem}

Note that the $p=2$ case is just Ridge regression, while the $p=\infty$ case is just scalar shrinkage of $\hat{\beta}_{OLS}$ towards 0, as per \citet{stein}. We refer to the $p=1$ case as Nuclear Regression and likewise Spectral Regression for the $p=\infty$ case. The exact relation between $C$ and $\alpha$ is given in \ref{alpha_vs_C}.

We also note that our theorem does not perfectly generalize the classical Gauss-Markov theorem. Indeed, the classical theorem is usually stated in the form $Var(\hat{\beta}_{OLS})\preceq Var(\hat{\beta})$ where $\hat{\beta}$ is any unbiased linear estimator \citep{wooldrige}. However, taking $C=0$ in our Theorem \ref{main_theorem}, we would conclude only that $Tr(Var(\hat{\beta}_{OLS}))\leq Tr(Var(\hat{\beta}))$, which is strictly weaker. But, we regard this as a bit of a technicality, and think our theorem faithfully captures the spirit of the original.

\begin{figure}[h]
\begin{center}
\includegraphics[width=14cm,height=3cm]{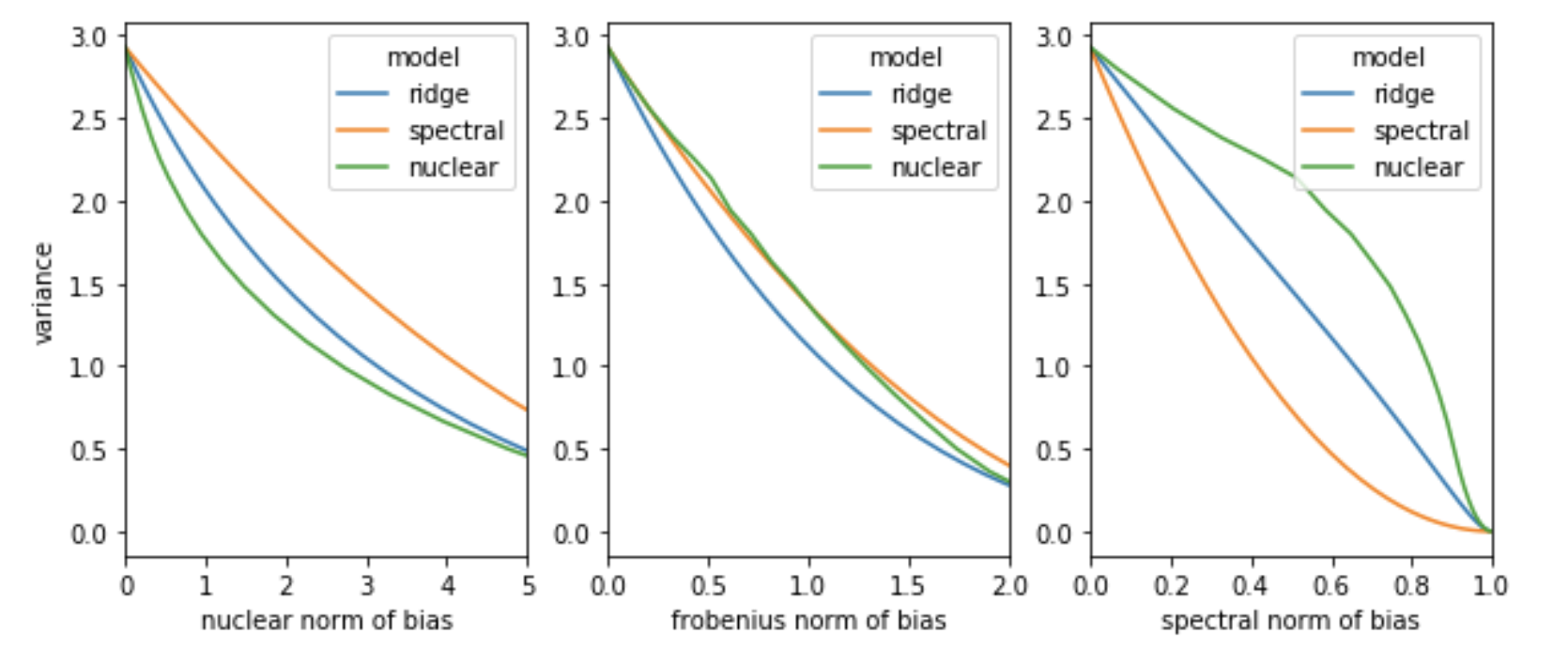}

\end{center}
\caption{Illustration of Theorem \ref{main_theorem} Each point on the curve corresponds to a regression fit with a certain value of $\alpha$, on the matrix $X=Diag(\sqrt{1},\sqrt{2},\dots,\sqrt{10})$. We show the norm of the bias matrix $B=LX-I$ on the x axis, and the sample variance $Tr(L^TL)$ on the y axis. The green curve always lies below the other two on the left plot, and analogously for the other two plots. }
\end{figure}
\subsection{Formulas for Test Error}
We next analytically characterize the generalization performance of the estimators derived in the previous section, under certain specific assumptions on the distribution of $X$. We will consider  two different random matrix ensembles: a standard spherical Gaussian ensemble, and another with orthogonal predictors but arbitrary spectrum. We will compute the test error averaged over random datasets, in an appropriate thermodynamic limit.

\subsubsection{Spherical Gaussian ensemble}
\label{subsec_gaussian}
Consider a predictor matrix $X\in\mathbb{R}^{N\times d}$, where each entry is an independent centered Gaussian of variance $1/N$. We consider the thermodynamic limit in which $N,d\to\infty$ and $d/N\to\lambda < 1$. The training targets are distributed as $Y_i=\langle X_i,\beta_0\rangle +\sigma \epsilon_i$, where $\epsilon_i\sim N(0,1)$, and the magnitude of the ground-truth coefficient vector $\beta_0$ satisfies $|\beta_0|^2/d\to \beta^2$, for some fixed $\beta>0$. 

The testing error for a given dataset $X,Y$ is defined as 

\begin{equation}\label{mse_spherical}MSE:=\mathbb{E}_{x\sim N(0,I_d/N)} (\langle x,\beta_0\rangle-\langle x,\hat{\beta}\rangle)^2,\end{equation}
where $\hat{\beta}$ is the coefficient vector estimated from $X$ and $Y$. (Note that we do not include exogenous noise in the testing data; if we did, it would just contribute an additive offset of $\sigma^2$). We observe that the expression can be simplified to $|\beta_0-\hat{\beta}|^2/N$, however we write it in the more verbose form above to make clear that it is obtained as an average over the distribution of the testing point $x$.

We are interested in the average value of this over different datasets. Using standard techniques in Random Matrix theory, it is not hard to show (cf. \ref{mse_proof}):

\begin{proposition}
\label{prop_gaussian_mse}
Consider the $p$-bias constrained estimator with regularization strength $\alpha$. 

In the above limit, the average test error $Err_p(\alpha):=\mathbb{E}_{X,Y} MSE$ is given by 
\begin{equation}Err(\alpha)=\lambda \int  \beta^2  (1-{\frac {x}{f_{\alpha}(x)}})^2+ \sigma^2  {\frac {x}{f_{\alpha}(x)^2}} \mu^{\lambda}_{MP}(dx)\end{equation}

where $\mu_{MP}^{\lambda}$ is the Marchenko-Pastur density with concentration parameter $\lambda$, and $f_{\alpha}$ is defined casewise as follows: 
\begin{equation}f_{\alpha}(x)=max(x,\alpha),p=1\end{equation}
\begin{equation}f_{\alpha}(x)=x+\alpha,p=2\end{equation}
\begin{equation}f_{\alpha}(x)=x(1+\alpha),p=\infty\end{equation}

\end{proposition}

In our later analyses, we will employ this formula, and evaluate the integral using numerical quadrature. However, it is interesting to note that the integral can be actually be further simplified in each of the three cases. For the ridge case, there is a well-known formula in terms of the Stieltjes transform of $\mu_{MP}^{\lambda}$ that easily follows from the above integral\citep{bai}. In the spectral case, we can derive a formula using simple calculus:
\begin{proposition}
For the Spectral estimator in the spherical Gaussian setup, the test error in the thermodynamic limit equals
\begin{equation}Err_{\infty}(\alpha)={\frac {\lambda\beta^2\alpha^2+{\frac {\lambda}{1-\lambda}}\sigma^2}{(1+\alpha)^2}}\end{equation}
\end{proposition}

To state the result for the Nuclear case,it is necessary to introduce the Appel hypergeometric function $F_1$\citep{appel}, which generalizes the classical Gauss hypergeometric function to two variables.

\begin{proposition}
\label{nuclear_gen_error}
For the Nuclear estimator in the spherical Gaussian setup, the test error in the thermodynamic limit can be expressed in terms of the Appel hypergeometric function $F_1$:
\[ Err_1(\alpha) = \begin{cases} 
      {\frac {\sigma^2\lambda}{1-\lambda}} & \alpha\leq \lambda_- \\
      {\frac {\sigma^2\lambda}{1-\lambda}}+\lambda\beta^2 CDF_{\lambda}(\alpha)+\sqrt{\alpha-\lambda_-}\sum_{r\in\{-1,1,2\}} c_r(\alpha) F_1({\frac 3 2},1-r,-{\frac 1 2},{\frac 5 2},1-{\frac {\alpha}{\lambda_-}},{\frac {\alpha-\lambda_-}{\lambda_+-\lambda_-}}) &  \lambda_-\leq \alpha\leq \lambda_+\\
      {\frac {\lambda(1+\lambda)}{\alpha^2}} +{\frac {\lambda \sigma^2-2\lambda \alpha\beta^2}{\alpha^2}}+\lambda\beta^2 & \alpha\geq \lambda_+
   \end{cases}
\]
where $\lambda_{\pm}$ are the limits of the Marchenko-Pastur support: $\lambda_{\pm}=(1\pm\sqrt{\lambda})^2$, $CDF_{\lambda}$ is the cumulative distribution function of $\mu_{MP}$ and $c_r(\alpha)$ are certain rational functions of $\alpha$.

\end{proposition}
See \ref{nuclear_gen_error_proof} for proof, and precise definition of $F_1$.
\subsubsection{Diagonal ensemble}
\label{subsec_diag}
While the spherical Gaussian case is an instructive starting point, the assumptions are clearly somewhat limiting. In particular, it implies a very specific functional form of the spectral density, which may not be a good match to real data. Thus, we analyze another random matrix ensemble which allows us the flexibility to specify an essentially arbitrary spectral density. On the one hand, this new ensemble has its own limiting assumptions that the Gram matrix $X^TX$ is almost surely diagonal, and the observations (i.e. rows of $X$) are no longer independent. But on the other, having another random matrix model allows us to assess whether qualitative observations about the spherical Gaussian case generalize to other settings.

We now describe the random matrix ensemble. We first specify some distribution $\nu$ supported on $[0,1]$, which will be the spectral density. We also specify some distribution $\nu_{noise}$ supported on a bounded interval of $\mathbb{R}_+$, which will act as multiplicative random noise on the spectrum; we require that $E_{x\sim \nu_{noise}}x=1$. To generate training and testing data, we first sample $\lambda_i\sim \nu, i=1,\dots, d$ independently, and $s_i\sim \nu_{noise}$. We then set 
\begin{equation}X_{tr}=X_1 diag(\sqrt{\lambda_is_i}), X_{test}= X_2 diag(\sqrt{\lambda_i})\end{equation}
where $X_1$ and $X_2$ are drawn independently from the Stiefel manifold $O(n,d)$ of $N\times d$ orthogonal matrices. That is $X_i^TX_i=I_d$. We observe that, as noted above: 1.) the Gram matrix $G=X_{tr}^TX_{tr}$ is diagonal almost surely, and 2.) the limiting eigenvalue density of $G$ is given by $\nu$. 

The regression targets are then constructed as before: 
$Y_{tr}=X_{tr}\beta_0+\epsilon, Y_{test}=X_{test}\beta_0$. The test error is defined as
\begin{equation}MSE:=N^{-1}|X_{test}\hat{\beta}-X_{test}\beta_0|^2 \end{equation}

Note that this is slightly different from the expression for MSE in the spherical Gaussian case (Equation \ref{mse_spherical}). The basic reason is that the observations (i.e. rows) are here not independent. Thus, in contrast to the spherical Gaussian case, where we could sample one test observation at a time and average, here we sample a batch of N (dependent) test observations at a time, and compute the average error over those N.

As before, we will be interested in the value of this error, marginalized over all of the randomness, and taken in the thermodynamic limit.

\begin{proposition}
\label{prop_mse_diag}
In the thermodynamic limit of the above random matrix ensemble, the average test error $Err_p(\alpha)=\mathbb{E}_{X_{tr},X_{test},\epsilon,s} MSE$ can be expressed as 
\begin{equation}Err_p(\alpha)=\lambda \int_0^1  \beta^2  x(1-{\frac {x}{f_{\alpha}(x)}})^2+ \sigma^2  {\frac {x^2}{f_{\alpha}(x)^2}} \nu(dx)\end{equation}
where $f_{\alpha}$ has the same form as in Proposition \ref{prop_gaussian_mse}. 
\end{proposition}
See \ref{mse_proof} for proof. Note the similarity to the result of Proposition \ref{prop_gaussian_mse}. The basic reason for this is that both ensembles have a similar rotational invariance, which allows the MSE to be expressed as a sum over the spectrum of $X_{tr}$. Besides the integrating measure, the only difference is an extra factor of $x$ in the Diagonal case.

In our experiments, we specialize to  power law distributions $d\nu(x)/dx\propto x^{\gamma-1}\textbf{1}_{x\in [0,1]}$, as many real covariance matrices have approximately this structure \citep{harris,qin,liu_pw,ruderman}. In this case, the expressions for the test error can again be simplified to analytic expressions that do not involve integrals; the details are straightforward and left to the interested reader\footnote{The Ridge case involves the Gauss hypergeometric function, the other two only require elementary calculus}.

\subsubsection{On the optimal Regularization form}
\label{opt_reg_form_section}
It is not hard to see that the derivations in the above two sections go through for any estimator of the form $\hat{\beta}=\hat{G}^{-1}X^TY$, where the eigenvalues of $\hat{G}$ are related to those of $G$ through some fixed point-wise transform $\lambda_i(\hat{G})=f(\lambda_i(G))$, and $\hat{G}$ is simultaneously diagonalizable with $G$. It is thus natural to ask what would be optimal estimator in this class, \textit{if we are allowed to select f arbitrarily}? By formally differentiating the integrand with respect to $f$ and setting equal to zero\footnote{this can be justified using results from Calculus of Variations}, it is easy to see that in both of the above random matrix models, the optimal form of $f$ corresponds to Ridge regression with regularization strength $\sigma^2/\beta^2$. 

Does this mean that Ridge regression will always perform better in practice than any other estimator? Not necessarily. For the above estimator is an \textit{oracle estimator} in the sense that it requires knowledge of the ground-truth $\beta$ and $\sigma$, which are typically not known. In practice, we would estimate the optimal ridge parameter using,e.g., cross-validation, and hope that our estimate is close to the oracle-optimal value $\sigma^2/\beta^2$. In the finite-sample regime, there will be some noise in the cross-validated errors, so this essentially amounts to trying to minimize $Err(\alpha)$ given only access to a small number of noisy estimates of the underlying function. Additionally, there will be ``resolution noise" caused by the fact that we will typically only try to estimate the test error for a small number of values of $\alpha$. A very rough rule of thumb is that if we sample $n$ values $\alpha_i$, all within some small distance $\delta$ to the minimum, then $min_i Err(\alpha_i)\approx \mu+{\frac {(\kappa\delta)^2}{(n+1)(n+2)}}$, where $\mu=\min_{\alpha}Err(\alpha)$ is the true minimum, and $\kappa^2$ is the Hessian  of $Err(\alpha)$ at the minimum (see \ref{curvature_rot} for justification). We conclude that the ``effective minimum" that we would find using cross validation is penalized by higher values of curvature, especially for small $n$.

\subsubsection{Experimental Validation}
In order to validate the above formulas, we generated synthetic datasets according to the random matrix ensembles from sections \ref{subsec_gaussian} and \ref{subsec_diag} , and compared the average test error with the theoretically predicted values. For each such dataset, we used $N=100$ training observations, and obtained the empirical test error by averaging over $100$ datasets. 
We did this for both the spherical Gaussian case, as well as the diagonal case with a power law spectral density $d\nu/dx\propto x^{\gamma-1}\textbf{1}_{x\in [0,1]}$. In Figure \ref{fig_pred_vs_emp} we show several plots with different values of the other hyperparameters; in all cases we see a very close correspondence between the empirical and predicted errors. In accordance with the discussion in the previous section, we can also see visually that the minima of the Nuclear estimator are often noticeably flatter than those of the Ridge, while also being only moderately higher. 
\begin{figure}[h]
\begin{center}
\includegraphics[width=10cm,height=3cm]{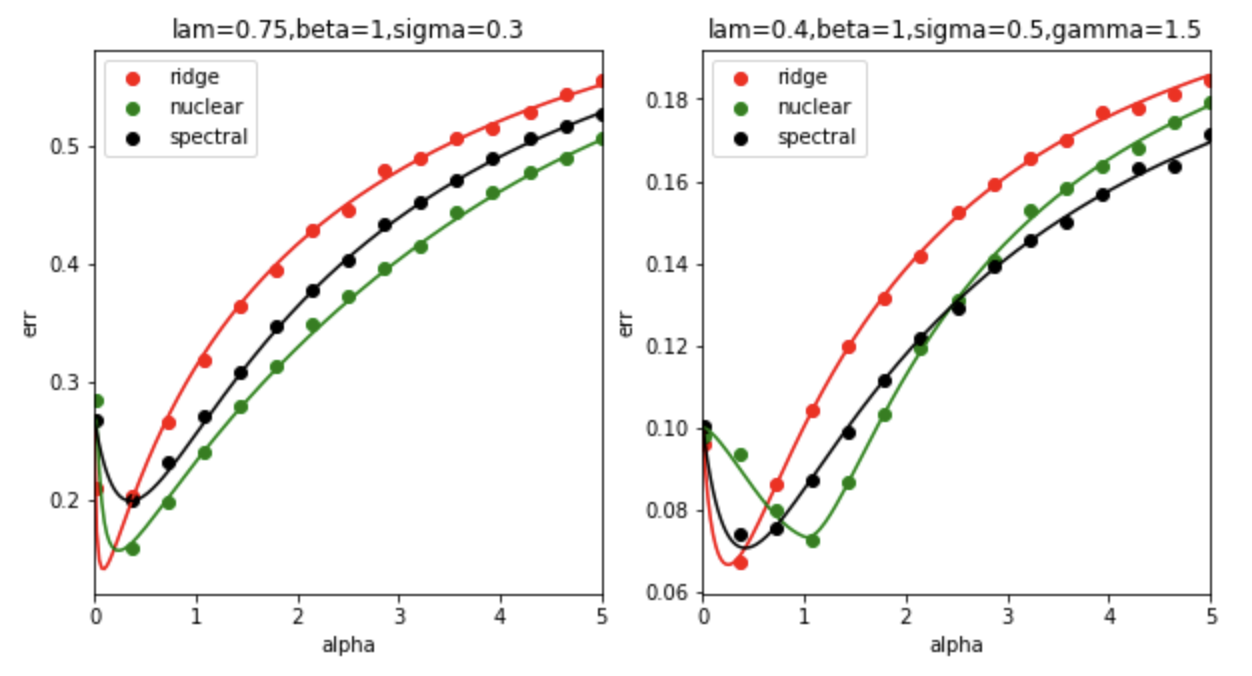}

\end{center}
\caption{Predicted vs. empirical test errors, for the spherical Gaussian case (left) and diagonal case with power law spectrum (right). Each dot represents an average over 20 different datasets, each with $N=100$ training observations.Note that the minima for the Ridge curve are notably steeper than for the other two.}
\label{fig_pred_vs_emp}
\end{figure}
\section{Experiments}
\subsection{Loss Basin Geometry Comparison}
Here, we aim to systematize the observation that the Nuclear estimator can have slightly higher, but much flatter, minima than the Ridge estimator. To do so, we simply fix the hyperparameters $\sigma,\lambda$, and consider the theoretical test error as a function of $\alpha$. We record both the minimal value over $\alpha$ $m:=min_{\alpha} Err(\alpha)$ as well as the curvature at the minimum: $\kappa:=\sqrt{\partial^2 Err(\alpha)/\partial\alpha^2}|_{\alpha=\alpha_{min}}$. For all cases, we fix $\beta=1$. For the sake of having an interpretable frame of reference, we report these values as percentage increases relative to the corresponding values for the Ridge estimator, shown in Figure \ref{gaussian_basin}. For the nuclear case, we can plainly see that the minima of the test loss are nearly as deep as for the the Ridge (within a few percentage points), however the minima are often very substantially flatter (corresponding to negative values in the table). The Spectral estimator tends to also have flatter minima than Ridge, but with a less favorable depth tradeoff compared to the Nuclear. Moreover, we see a pronounced effect of $\sigma$, with increasing values tending to decrease the gap between the depths of Nuclear and Ridge minima, while also decreasing the difference in curvature. As in Figure \ref{diag_basin}, we see that these observations largely carry over to the case of a power-law spectral density, indicating that they are not simply an incidental property of the spherical Gaussian model. See \ref{min_estimation} for more technical details on the calculation of these values.
\begin{figure}[h]
\begin{center}
\includegraphics[width=16cm]{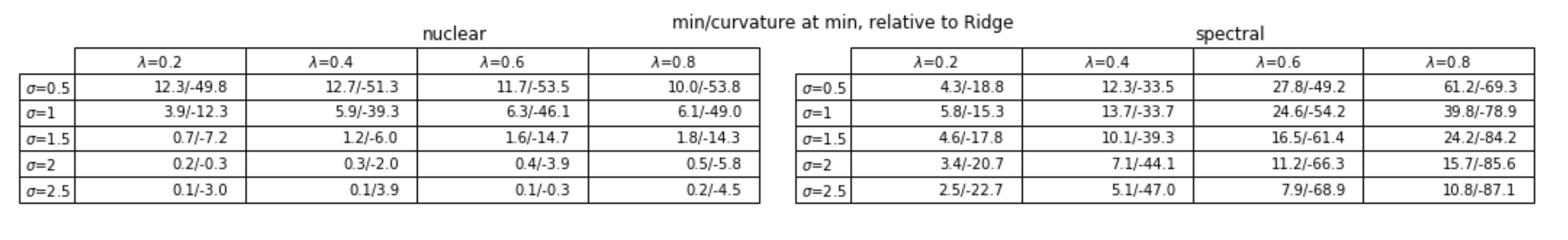}

\end{center}
\caption{Quantification of loss basin geometry for the spherical Gaussian model. Each entry in the table shows the minimal attainable test error/curvature at the minimum, where each is expressed as a percentage increase relative to the corresponding value for the Ridge estimator.}
\label{gaussian_basin}
\end{figure}

\begin{figure}[h]
\begin{center}
\includegraphics[width=14cm,height=9cm]{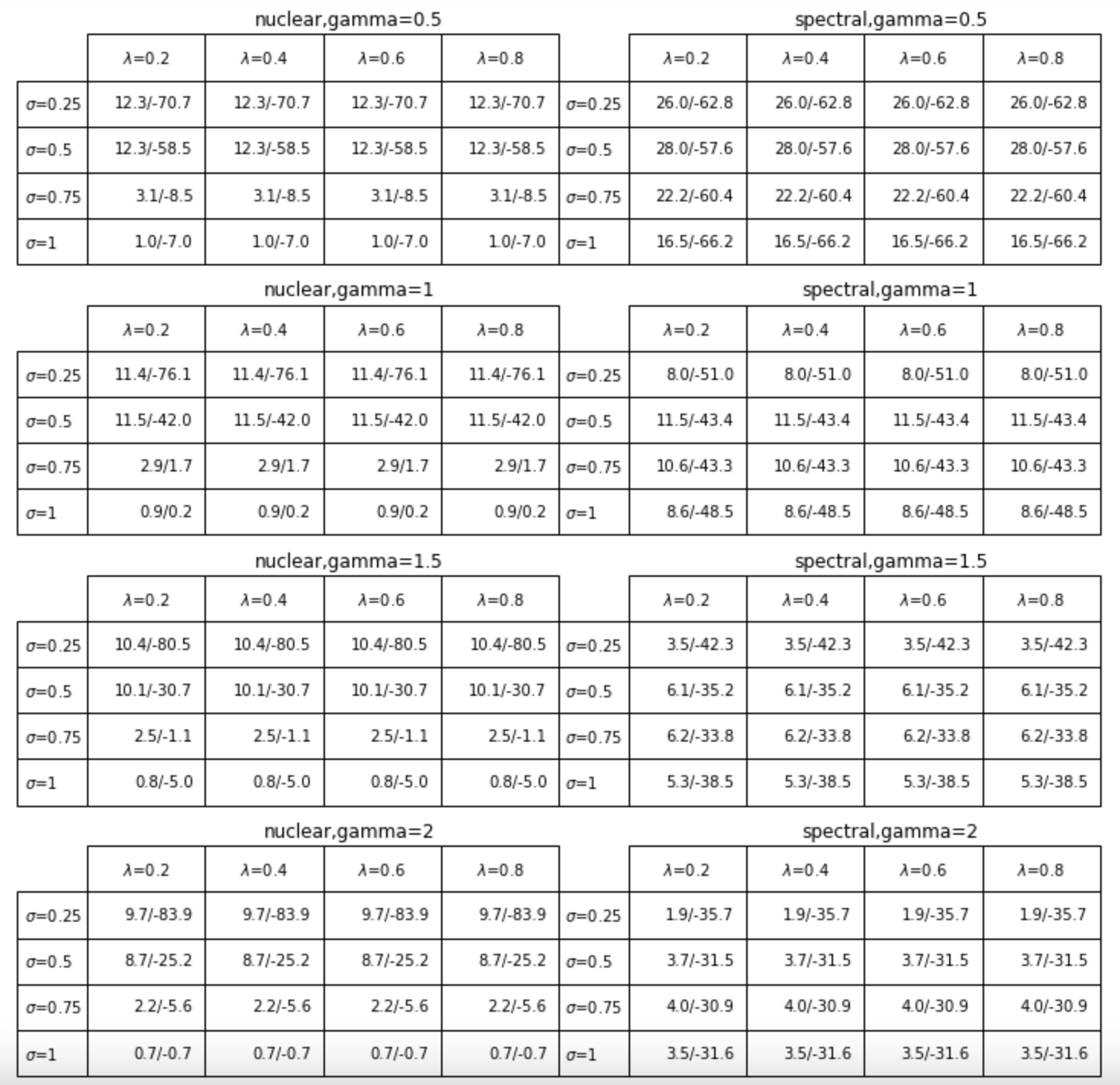}

\end{center}
\caption{Same as Figure \ref{gaussian_basin} except for the Diagonal matrix model with power law spectral density.}
\label{diag_basin}
\end{figure}

\subsection{Simulation Studies}
The purpose of this section is to show that the observations about the loss basin geometry can have practical consequences in settings where the regularization strength is selected by cross-validation. 
\subsubsection{Gaussian Predictors}
\label{gaussian_predictor}
We generated training data by sampling $X_i\sim N(0,C(\rho))$, where $C(\rho):=(1-\rho)I_d+\rho\textbf{1}\textbf{1}^t$. We then generated a ground truth coefficient vector by sampling $\beta_0\sim N(0,I_d)$, and training targets $y_i$ by $y_i\sim N(\langle X_i,\beta_0\rangle,\sigma^2)$. The testing data were generated similarly, using the same $\beta_0$. In all cases, we used $N=100$ training examples, and $5000$ testing examples. For each combination of hyperparameters $d,\sigma,\rho$, we generated 100  train/test datasets. For each model (Spectral, Ridge, Nuclear) and each individual dataset, we used 3-fold cross validation to select the best-performing $\alpha$ on the training set. We then refit the model on the entire training set, and evaluated its performance on the test set. The  set of allowable $\alpha$ values considered in the cross-validation was the same for all models, and consisted of 9 equally logarithmically spaced values spanning $10^{-4}$ to $10^6$.

In Figure \ref{fig_cv} we plot the best-performing model for each combination of hyperparameters, for two different definitions of ``best". In the top row, we show the model that attains the lowest average test error in each cell. In the bottom, we show the model that is most likely to ``win" on any given dataset. That is, suppose we evaluate model $m$ on dataset $i$, and obtain a test error of $MSE_{mi}$. The top row shows $argmin_m \mathbb{E}_jMSE_{mj}$, while the bottom row shows $Mode(\{argmin_m MSE_{mj}\}_j)$. We see that in terms of average error, the Ridge and Nuclear are essentially matched, with Spectral being a distant 3rd. The Nuclear however is overwhelmingly likely to be the winner on a given dataset.

\begin{figure}[h]
\begin{center}
\includegraphics[width=16cm,height=4cm]{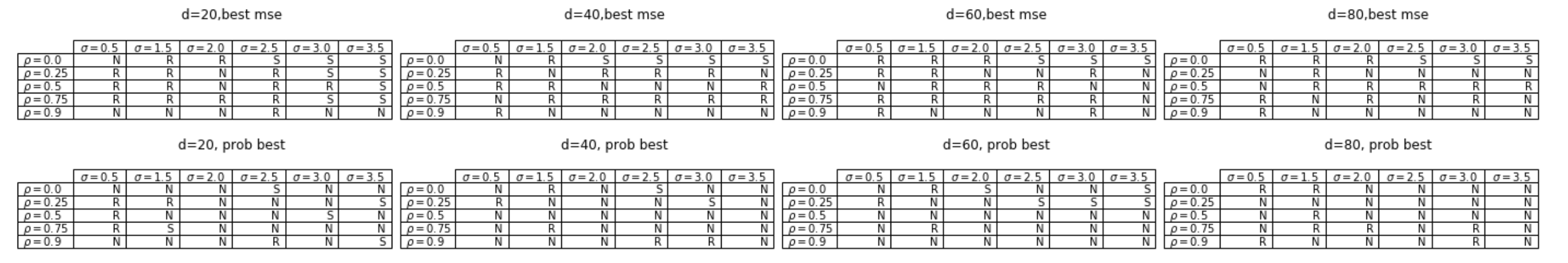}

\end{center}
\caption{Best performing models with cross-validation on Gaussian predictors. Top: best average error. Bottom: highest probability of winning. Each cell is an aggregate over 100 datasets.}
\label{fig_cv}
\end{figure}

\subsubsection{Regression of Nonlinear functions using Random Fourier features}
In order to generalize our observations beyond the Gaussian features case, we consider a more complex setup in which the true relationship between the predictors and the targets is non-linear, and the regression is performed in the feature space of a universal approximating kernel. To be more specific, the process for generating data was as follows. We first pick integers $d,d_{rbf}$ which define the dimension of the observed space, and of the kernel feature map, respectively. We sampled a training matrix $X$ of size $N\times d$, in which each entry has variance $1/d_{rbf}$. We then sampled a non-linear function $f:\mathbb{R}^d\to\mathbb{R}$ in a manner described below, and generated training targets as $y_i=f(X_i)+\sigma\epsilon_i$, where $X_i$ is the ith row of $X$ (i.e. ith observation) and $\epsilon_i$ is unit Gaussian. Testing data was generated similarly, using the very same function $f$ (and setting $\sigma=0$ as in the previous analyses). Since the relation between $X$ and $Y$ is non-linear, we do not directly fit the models on $X$, but rather first fit a Random Fourier features (RFF) model \citep{rahimi} on $X$, using $d_{rbf}$  features \footnote{The RFF fit was done using the class \texttt{sklearn.kernel\char`_approximation.RBFSampler}.}.We then fit the linear models on $RFF(X),Y$, where $RFF(X)$ is the $N\times d_{rbf}$ matrix obtained by applying the Random Fourier features mapping. To estimate test error, we first transform $X_{test}$ using the \textit{same} RFF mapping, and consider $|RFF(X_{test})\hat{\beta}-Y_{test}|^2$. Finally, the non-linear function $f$ was defined as follows:
\begin{equation}f(x)=\sum_{k=1}^{d_{rbf}}cos(2\pi k\langle x,v_i\rangle)/k^2\end{equation}
where $v_i$ are sampled uniformly from the unit sphere $S^{d-1}\subset\mathbb{R}^d$. We fixed $d=10$ and $N=100$ in all cases. For each value of $d_{rbf}$ and $\sigma$, we generated 100 datasets in this way. We did not consider the Spectral estimator for this experiment, because it does not naturally accomodate the overparametrized case.As in Figure \ref{fig_cv_rff}, we see that the cross-validated Nuclear estimator can again outperform the cross-validated Ridge, with the effect becoming more pronounced for larger values of $\sigma$.

\begin{figure}[h]
\begin{center}
\includegraphics[width=12cm,height=4cm]{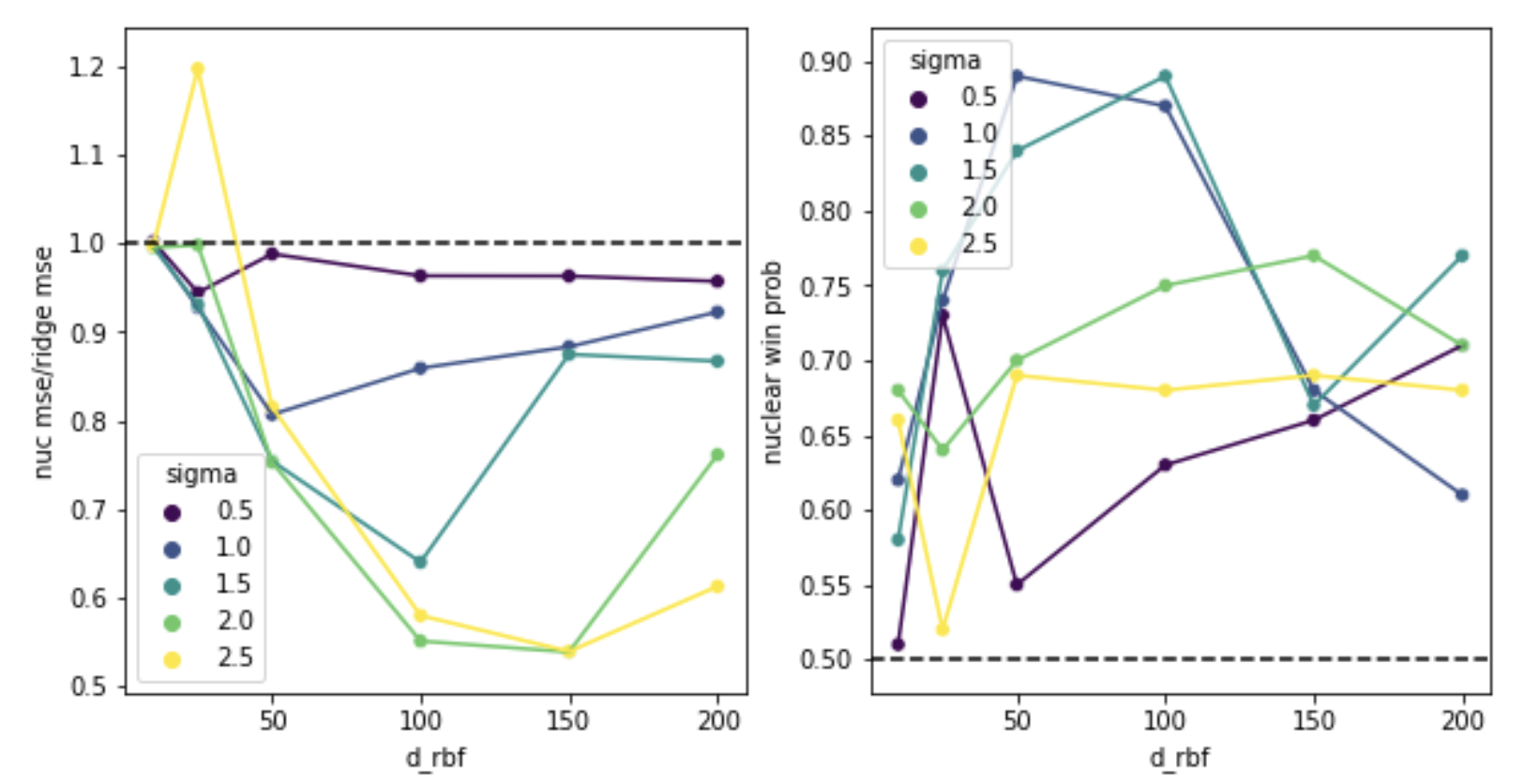}

\end{center}
\caption{Test set performance for cross-validated Nuclear estimator on the Random Fourier features data. Left shows average test error expressed as a ratio relative to average Ridge test error. Right shows probability of ``winning" on a given dataset. Each dot represents an aggregate over 100 datasets.}
\label{fig_cv_rff}
\end{figure}

\section{Related Work}

Our work relates to several distinct areas. On the one hand, there are questions from classical statistics. For instance, several other sorts of generalized Gauss-Markov theorems have appeared in the literature, although they generalize along different axes than our formulation. for example allowing non-spherical errors \citep{kourouklis},predictor collinearity \citep{lewis}, random effects \citep{shaffer},or  non-linear estimators \citep{hansen} 
.Similarly,the Nuclear and Spectral estimators bear a very strong resemblance to various notions from classical statistics, most notably Principal Components Regression (PCR) \citep{kendall_pcr}, Stein Shrinkage \citep{stein} respectively. Finally, the work of \citet{hocking} is particularly relevant and worthy of further discussion here-we give a more detailed comparison in \ref{hocking_comp}.

Lately, there there has also been renewed interest in the study of high-dimensional regression models and the effects of various kinds of regularization thereupon. Most such works focus on Ridge or kernel Ridge regression (\citep{mei,dobriban,canatar,dicker,belkin,hu_krr,nakkiran}). Ridge regression has also proven a powerful lens to understand other techniques such as  dropout \citep{wager}, data augmentation\citep{lin} and early stopping \citep{ali}. This pervasiveness of the ridge underscores our claim that a detailed understanding of other kinds of regularization than L2 in the linear case could be used to study more complex regularizations in non-linear settings. 

Some other work has analyzed other regularizations than ridge such as \citet{bayati,samet}for Lasso. Theoretical properties of nuclear norm have mainly been studied in the context of matrix completion, in which it is often used as a surrogate for matrix rank \citep{candes,koltchinskii}, but see \citet{hu} for other applications. Another setting in which rank penalties appear in regression problems is Low Rank Regression \citep{bunea,izenman}, although this is fairly different than our setting as it has to do with regularizing the \textit{outputs} of a regression problem with multiple dependent variables.There is also some modern work on theoretical properties of the Principal Components Regression estimator \citep{xu_pcr}; optimizing the number of PCR components has a similar flavor to optimizing $\alpha$ in the Nuclear estimator.

Another line of related work is that while we have here studied the behavior of flat minima in the context of linear models and cross-validation, this same concept has also been quite popular as an explanatory tool in the context of generalization of neural networks \citep{baldassi,mulayoff,hochreiter}, in which the randomness in the minimization procedure comes from the stochasticity of gradient descent. 
\section{Discussion}
\label{discussion_section}
We have introduced two simple Linear models:Spectral and Nuclear regression, and have shown how the forms of both of these models, together with Ridge regression, may all be derived from an extension of  Gauss-Markov Theorem for Schatten norm constraints on the bias operator. By considering theoretical error curves as a function of regularization strength, we observed in multiple different random matrix ensembles that the Nuclear model can often attain minima that are much flatter than those of the Ridge model, while being nearly as deep. Finally, we showed that this tradeoff can have practical consequences in a cross-validation setting, for both Gaussian and non-Gaussian features.

It should be noted that the particulars of the cross-validation results depend on the number of $\alpha$ values used in the grid search, among other factors; and as noted before, the effect of the curvature at the minimum decreases as more values of $\alpha$ are used. However, increasing the number of $\alpha$ values has other potential costs such as risk of overfitting or increased computational cost. But more fundamentally,our intention with the simulation results is not to claim that the Nuclear estimator should always be preferred over the Ridge, but rather to show that it \textit{can} outperform in some practically relevant setups, and to use our theoretical analysis to give a precise characterization of why this happens.

There are many potential future directions of this work, such as further understanding of the Spectral and Nuclear regularizations in the kernelized or overparametrized regime, or using insights from the Nuclear case to understand other situations that involve rank constraints such as matrix completion. More broadly, we hope that this framework will provide a test case to understand less common regularizations such as $\|\cdot\|_1$ and $\|\cdot\|_{\infty}$, with the goal of transferring the insights gained thereby to the study of high-dimensional non-linear models.

\newpage
\bibliography{iclr2024_conference}

\begin{thebibliography}{49}
\providecommand{\natexlab}[1]{#1}
\providecommand{\url}[1]{\texttt{#1}}
\expandafter\ifx\csname urlstyle\endcsname\relax
  \providecommand{\doi}[1]{doi: #1}\else
  \providecommand{\doi}{doi: \begingroup \urlstyle{rm}\Url}\fi

\bibitem[Ali et~al.(2019)Ali, Kolter, and Tibshirani]{ali}
A~Ali, JZ~Kolter, and RJ~Tibshirani.
\newblock A continuous-time view of early stopping for least squares.
\newblock \emph{AISTATS}, 2019.

\bibitem[Appell(1925)]{appel}
P~Appell.
\newblock Sur les fonctions hypergéométriques de plusieurs variables.
\newblock \emph{Mémoir. Sci. Math. Paris: Gauthier-Villars}, 1925.

\bibitem[Bai \& Silverstein(2010)Bai and Silverstein]{bai}
Z~Bai and J~Silverstein.
\newblock \emph{Spectral Analysis of Large Dimensional Random Matrices}.
\newblock Springer, 2010.

\bibitem[Bailey(1934)]{bailey}
WN~Bailey.
\newblock On the reducibility of appell's function $f_4$.
\newblock \emph{Quart. J. Math}, 1934.

\bibitem[Baldassi et~al.(2021)Baldassi, Lauditi, Malatesta, Perugini, and Zecchina]{baldassi}
C~Baldassi, C~Lauditi, EM~Malatesta, G~Perugini, and R~Zecchina.
\newblock Unveiling the structure of wide flat minima in neural networks.
\newblock \emph{Phys Rev Letters}, 2021.

\bibitem[Bayati \& Montanari(2011)Bayati and Montanari]{bayati}
M~Bayati and A~Montanari.
\newblock The dynamics of message passing on dense graphs, with applications to compressed sensing.
\newblock \emph{IEEE Transcations on Information Theory}, 2011.

\bibitem[Belkin et~al.(2019)Belkin, Hsu, Ma, and Mandal]{belkin2}
M~Belkin, D~Hsu, S~Ma, and S~Mandal.
\newblock Reconciling modern machine learning practice and the bias-variance trade-off.
\newblock \emph{Proc Nat Aca Sci}, 2019.

\bibitem[Belkin et~al.(2020)Belkin, Hsu, and Xu]{belkin}
M~Belkin, D~Hsu, and J~Xu.
\newblock Two models of double descent for weak features.
\newblock \emph{preprint arXiv:1903.07571}, 2020.

\bibitem[Bunea et~al.(2011)Bunea, She, and Wegkamp]{bunea}
F~Bunea, Y~She, and M~Wegkamp.
\newblock Optimal selection of reduced rank estimators of high-dimensional matrices.
\newblock \emph{Ann. Statist.}, 2011.

\bibitem[Canatar et~al.(2021)Canatar, Bordelon, and Pehlevan]{canatar}
A~Canatar, B~Bordelon, and C~Pehlevan.
\newblock Spectral bias and task-model alignment explain generalization in kernel regression and infinitely wide neural networks.
\newblock \emph{Nature communications}, 2021.

\bibitem[Candès \& Recht(2012)Candès and Recht]{candes}
E~Candès and B~Recht.
\newblock Exact matrix completion via convex optimization.
\newblock \emph{Communications of the ACM}, 2012.

\bibitem[Dicker(2016)]{dicker}
LH~Dicker.
\newblock Ridge regression and asymptotic minimax estimation over spheres of growing dimension.
\newblock \emph{International Statistical Institute and the Bernoulli Society for Mathematical Statistics and Probability}, 2016.

\bibitem[Dobriban \& Wager(2018)Dobriban and Wager]{dobriban}
E~Dobriban and S~Wager.
\newblock High-dimensional asymptotics of prediction: Ridge regression and classification.
\newblock \emph{Annals of statistics}, 2018.

\bibitem[Fan(1951)]{kyfan}
Ky~Fan.
\newblock Maximum properties and inequalities for the eigenvalues of completely continuous operators.
\newblock \emph{Proc. Nat. Acad. Sci. U.S.A.}, 1951.

\bibitem[Goodfellow et~al.(2016)Goodfellow, Bengio, and Courville]{goodfellow}
I~Goodfellow, Y~Bengio, and A~Courville.
\newblock \emph{Deep learning}.
\newblock MIT press, 2016.

\bibitem[Hansen(2022)]{hansen}
BE~Hansen.
\newblock A modern gauss–markov theorem.
\newblock \emph{Econometrica}, 2022.

\bibitem[Hastie et~al.(2019)Hastie, Montanari, Rosset, and Tibshirani]{hastie}
T~Hastie, A~Montanari, S~Rosset, and RJ~Tibshirani.
\newblock Surprises in high-dimensional ridgeless least squares interpolation.
\newblock \emph{arxiv:abs/1903.08560}, 2019.

\bibitem[Hochreiter \& Schmidhuber(1997)Hochreiter and Schmidhuber]{hochreiter}
S~Hochreiter and J~Schmidhuber.
\newblock Flat minima.
\newblock \emph{Neural Computation}, 1997.

\bibitem[Hocking et~al.(1976)Hocking, Speed, and Lynn]{hocking}
R.~R. Hocking, F.~M. Speed, and M.~J. Lynn.
\newblock A class of biased estimators in linear regression.
\newblock \emph{Technometrics}, 1976.

\bibitem[Hu \& Lu(2022)Hu and Lu]{hu_krr}
H~Hu and YM~Lu.
\newblock Sharp asymptotics of kernel ridge regression beyond the linear regime.
\newblock \emph{arXiv:2205.06798}, 2022.

\bibitem[Hu et~al.(2021)Hu, Nie, Wang, and Li]{hu}
Z~Hu, F~Nie, R~Wang, and X~Li.
\newblock Low rank regularization: A review.
\newblock \emph{Neural Networks}, 2021.

\bibitem[Izenman(2008)]{izenman}
AJ~Izenman.
\newblock Modern multivariate statistical techniques: Regres- sion, classification, and manifold learning.
\newblock \emph{springer}, 2008.

\bibitem[Jacot et~al.(2020)Jacot, Simsek, Spadaro, Hongler, and Gabriel]{jacot}
A~Jacot, B~Simsek, F~Spadaro, C~Hongler, and F~Gabriel.
\newblock Kernel alignment risk estimator: risk prediction from training data.
\newblock \emph{In Advances in Neural Information Processing Systems}, 2020.

\bibitem[Johansson et~al.(2022)Johansson, Engsner, Strannegard, and Mostad]{johansson}
A~Johansson, N~Engsner, C~Strannegard, and P~Mostad.
\newblock Exact spectral norm regularization for neural networks.
\newblock \emph{arxiv:2206.13581.pdf}, 2022.

\bibitem[Kamalakara et~al.(2022)Kamalakara, Locatelli, Venkitesh, Ba, Gal, and Gomez]{rao}
S~Rao Kamalakara, A~Locatelli, B~Venkitesh, J~Ba, Y~Gal, and AN~Gomez.
\newblock Exploring low rank training of deep neural networks.
\newblock \emph{arXiv:2209.13569}, 2022.

\bibitem[Kendall(1957)]{kendall_pcr}
MG~Kendall.
\newblock \emph{A Course in Multivariate Analysis}.
\newblock London: Charles Griffin, 1957.

\bibitem[Koltchinskii et~al.(2011)Koltchinskii, Lounici, and Tsybakov]{koltchinskii}
V~Koltchinskii, K~Lounici, and AB~Tsybakov.
\newblock Nuclear-norm penalization and optimal rates for noisy low-rank matrix completion.
\newblock \emph{Ann. Statist}, 2011.

\bibitem[Kourouklis \& Paige(1981)Kourouklis and Paige]{kourouklis}
S~Kourouklis and CC~Paige.
\newblock A constrained least squares approach to the general gauss- markov linear model.
\newblock \emph{Journal of American Statistical Assoc.}, 1981.

\bibitem[Lewis \& Odell(1966)Lewis and Odell]{lewis}
TO~Lewis and PL~Odell.
\newblock A generalization of the gauss-markov theorem.
\newblock \emph{Journal of American statistical assoc.}, 1966.

\bibitem[Lin et~al.(2023)Lin, Kaushik, Dyer, and Muthukumar]{lin}
CH~Lin, C~Kaushik, EL~Dyer, and V~Muthukumar.
\newblock The good, the bad and the ugly sides of data augmentation: An implicit spectral regularization perspective.
\newblock \emph{arXiv:2210.05021}, 2023.

\bibitem[Liu et~al.(1997)Liu, Cizeau, Meyer, Peng, and Stanley]{liu_pw}
Y~Liu, P~Cizeau, M~Meyer, CK~Peng, and HE~Stanley.
\newblock Correlations in economic time series.
\newblock \emph{Physica A: Statistical Mechanics and its Applications}, 1997.

\bibitem[Mei \& Montanari(2020)Mei and Montanari]{mei}
S~Mei and A~Montanari.
\newblock The generalization error of random features regression: Precise asymptotics and double descent curve.
\newblock \emph{arxiv.org/pdf/1908.05355.pdf}, 2020.

\bibitem[Mianjy \& Arora(2019)Mianjy and Arora]{mianjy}
P~Mianjy and R~Arora.
\newblock On dropout and nuclear norm regularization.
\newblock \emph{International conference on machine Learning}, 2019.

\bibitem[Mulayoff \& Michaeli(2020)Mulayoff and Michaeli]{mulayoff}
R~Mulayoff and T~Michaeli.
\newblock Unique properties of flat minima in deep networks.
\newblock \emph{Proceedings of the 37th International Conference on Machine Learning}, 2020.

\bibitem[Nakkiran et~al.(2021)Nakkiran, Venkat, Kakade, and Ma]{nakkiran}
P~Nakkiran, P~Venkat, SM~Kakade, and T~Ma.
\newblock Optimal regularization can mitigate double descent.
\newblock \emph{In International Conference on Learning Representations}, 2021.

\bibitem[Qin \& Colwell(2018)Qin and Colwell]{qin}
C~Qin and LJ~Colwell.
\newblock Power law tails in phylogenetic systems.
\newblock \emph{Proc. Nat. Aca. Sci}, 2018.

\bibitem[Rahimi \& Recht(2008)Rahimi and Recht]{rahimi}
A~Rahimi and B~Recht.
\newblock Weighted sums of random kitchen sinks: Replacing minimization with randomization in learning.
\newblock In \emph{NeurIPS}, 2008.

\bibitem[Ruderman \& Bialek(1994)Ruderman and Bialek]{ruderman}
DL~Ruderman and W~Bialek.
\newblock Statistics of natural images: Scaling in the woods.
\newblock \emph{Physical review letters}, 1994.

\bibitem[Samet et~al.(2013)Samet, Thrampoulidis, and Hassibi]{samet}
O~Samet, C~Thrampoulidis, and B~Hassibi.
\newblock The squared-error of generalized lasso: A precise analysis.
\newblock \emph{Allerton Conference on Communication, Control, and Computing}, 2013.

\bibitem[Saxe et~al.(2013)Saxe, McClelland, and Ganguli]{saxe}
AM~Saxe, JL~McClelland, and S~Ganguli.
\newblock Exact solutions to the nonlinear dynamics of learning in deep linear neural networks.
\newblock \emph{arXiv:1312.6120}, 2013.

\bibitem[Shaffer(1991)]{shaffer}
JP~Shaffer.
\newblock The gauss—markov theorem and random regressors.
\newblock \emph{The American Statistician}, 1991.

\bibitem[Srivastava et~al.(2014)Srivastava, Hinton, Krizhevsky, Sutskever, and Salakhutdinov]{srivastava}
N~Srivastava, G~Hinton, A~Krizhevsky, I~Sutskever, and R~Salakhutdinov.
\newblock Dropout: A simple way to prevent neural networks from overfitting.
\newblock \emph{The Journal of Machine Learning Research}, 2014.

\bibitem[Stein(1962)]{stein}
J~Stein.
\newblock Multiple regression.
\newblock \emph{Contributions to Probability and Statistics. Essays in Honor of Harold Hotelling}, pp.\  424--443, 1962.

\bibitem[Stringer et~al.(2019)Stringer, Pachitariu, Steinmetz, Carandini, and Harris]{harris}
C~Stringer, M~Pachitariu, N~Steinmetz, M~Carandini, and KD~Harris.
\newblock High-dimensional geometry of population responses in visual cortex.
\newblock \emph{Nature}, 2019.

\bibitem[Wager et~al.(2013)Wager, Wang, and Liang]{wager}
S~Wager, S~Wang, and P~Liang.
\newblock Dropout training as adaptive regularization.
\newblock \emph{arXiv:1307.1493v2}, 2013.

\bibitem[Wooldride(2015)]{wooldrige}
JM~Wooldride.
\newblock \emph{Introductory Econometrics: A Modern Approach}.
\newblock Cengage Learning, 2015.

\bibitem[Xu \& Hsu(2019)Xu and Hsu]{xu_pcr}
J~Xu and D~Hsu.
\newblock On the number of variables to use in principal component regression.
\newblock \emph{Neurips}, 2019.

\bibitem[Yang et~al.(2020)Yang, Tang, Wen, Yan, Hu, Li, Li, and Chen]{yang}
H~Yang, M~Tang, W~Wen, F~Yan, D~Hu, A~Li, H~Li, and Y~Chen.
\newblock Learning low-rank deep neural networks via singular vector orthogonality regularization and singular value sparsification.
\newblock \emph{arXiv:2004.09031}, 2020.

\bibitem[Yoshida \& Miyato(2017)Yoshida and Miyato]{yoshida}
Y~Yoshida and T~Miyato.
\newblock Spectral norm regularization for improving the generalizability of deep learning.
\newblock \emph{arXiv:1705.10941}, 2017.

\end{thebibliography}
\bibliographystyle{iclr2024_conference}

\newpage
\appendix
\section{appendix}
\subsection{Code availability}
Code for the simulations is available at \url{https://github.com/SimonSegert/specreg}.
\subsection{Further comparison with Hocking et. al.}
\label{hocking_comp}
The paper, similarly to us in spirit, introduces a class of biased linear estimators defined by a family of constrained optimization problems, and derive in this way the forms of Ridge, Stein Shrinkage, and PCR regressors. However, a closer inspection reveals that their setup is considerably different from ours. 

For simplicity, we will make the assumption that $G=X^TX$ is diagonal, as this is sufficient to illustrate the differences. The form of estimators considered in Hocking et. al. is
\begin{equation}\hat{\beta}=v\circ \Sigma\end{equation}
where as before, $\Sigma$ is the vector of eigenvalues of $G$ and $\circ$ is Hadamard (element-wise) product. Immediately we see a difference that their setup has only $d$ free variables, from the vector $p$, whereas we have $N\times p$ free variables in the matrix $L$. 

By a bit of algebra,
\begin{eqnarray}
\hat{\beta}&=&\Sigma \circ v=(\Sigma\circ v\oslash X^TY)\circ X^TY\\
& = & diag(\Sigma)diag(v)diag(X^TY)^{-1}X^TY\\
&:=& DX^TY 
\end{eqnarray}
where $\oslash$ is element-wise division, and $D$ is a diagonal matrix. Thus they effectively assume from the outset that $\hat{\beta}=DX^TY$ for some diagonal matrix $D$, whereas we a priori allow $\hat{\beta}=LY$ for $L\in\mathbb{R}^{N\times d}$ and \textit{derive} that $L$ is necessarily of this form. 

They then show that each of the aforementioned estimators can be derived as the solution to an optimization of the form $\min_{v\in C} L(v)$, where $L$ is a loss function and $C$ is a constraint set, that depends on the estimator. The forms of the constraint sets are rather \textit{ad-hoc} and the authors do not appear justify the form from more basic principles, or to relate them to each other in a meaningful way. By contrast, the forms of our constraints arise naturally from a standard family of matrix norms. 

Finally, they do not precisely characterize the bias of each of the estimators, whereas our results characterize the estimators as being optimal for a fixed level of bias.

 \subsection{Rule of Thumb for depth vs. curvature}
\label{curvature_rot}
In order to get insight into the tradeoff between curvature and depth, we consider a highly simplified setup which nonetheless captures some key features of selecting $\alpha$ through cross-validation. Let us model the function of average test error vs regularization strength as a parabola $f(x)=\kappa^2x^2/2+
\mu$, for $-\delta\leq x\leq \delta$. Thus $\kappa$ controls the curvature and $\mu$ controls the depth of the minimum. We will model cross-validation as taking n iid samples 
$X_i\sim Unif(-\delta,\delta)$ and then estimating the minimum of $f$ as $\hat{\mu}=min_i(f(X_i))$. Note that in a real cross-validation setup there is the additional complication that we cannot perfectly observe the value $f(x)$ due to finite-sample variability; we do not model this effect here. 

What is the distribution of $\hat{\mu}$? Well, for $\mu<y<\delta^2\kappa^2/2+\mu=f_{max}$, we have
\begin{equation}\mathbb{P}(f(X_i)>y)=\mathbb{P}(|X_i|>\sqrt{2(y-\mu)}/\kappa={\frac 1 {2\delta}}(2\delta-2\sqrt{2(y-\mu)}/\kappa)\end{equation}
So, 
\begin{equation}\mathbb{P}(\hat{\mu}>y)=\mathbb{P}(f(X_i)>y)^n=(1-\delta^{-1}\sqrt{2(y-\mu)}/\kappa)^n\end{equation}

The expected value is 
\begin{equation}\mathbb{E} \hat{\mu}=\int_0^{\infty} \mathbb{P}(\mu>y)dy=\mu+\int_{\mu}^{\delta^2\kappa^2/2+\mu}(1-\sqrt{(2(y-\mu)\kappa^{-2}\delta^{-2}})^ndy\end{equation}
Letting $u=\sqrt{2(y-\mu)/(\kappa^2\delta^2)}$,the integral becomes 

\begin{equation}\mu+\kappa^2\delta^2\int_{0}^1 (1-u)^nudu=\mu+\kappa^2\delta^2Beta(2,n+1)=\mu+{\frac {(\kappa\delta)^2}{(n+1)(n+2)}}\end{equation}

In the large sample limit, the contribution of the curvature vanishes, and we can exactly find the minimum. However, for finite samples this is not the case, and it could be that for two different parabolas the true minima satisfy $\mu_1<\mu_2$ while the estimated minima satisfy $\mathbb{E}\hat{\mu_1}>\mathbb{E}\hat{\mu_2}$.

\subsection{proof of main theorem}
Let us first make the simple but important observation that a minimizer in Equation \ref{main_objective} actually exists (since the constraint set is non-compact this is not completely immediate). One way to see this is to note that the problem is of the form $\min_{x\in D}|x|^2$ for some closed set $D$, and problems of this form always have a minimizer (even if $D$ is non-compact). Combined with the observation that the objective function is strictly convex, we conclude that the problem in Equation \ref{main_objective} has \textit{exactly one } minimizer.

The basic strategy of the proof is now to derive certain properties of the minimizer, and add these properties as further constraints until we get something tractable.

\begin{lemma}
The minimizer $L_{opt}$ of the bias-constrained problem (Equation \ref{main_objective}) takes the form $QX^T$ for some $d\times d$ matrix $Q$
\end{lemma}
\begin{proof}
Let $X_{\perp}$ be any $N\times (N-d)$ matrix whose columns for a basis for the orthogonal complement of $Colspace(X)$. We may assume that the columns are orthonormal; $X_{\perp}^TX_{\perp}=I_{N-d}$. Note that the concatenated matrix $[X,X_{\perp}]$ is invertible. Therefore, we can express the optimum as $L=M\left(\begin{array}{c}X^T\\X_{\perp}^T\end{array}\right)$ for some $d\times N$ matrix $M$. Writing $M$ in block form $M=[Q,Q_{\perp}]$ where $Q\in\mathbb{R}^{d\times d}$ and $Q_{\perp}\in\mathbb{R}^{d\times (N-d)}$, we have $L=QX^T+Q_{\perp}X_{\perp}^T$. Note that $X^TX_{\perp}=0$; therefore 
\begin{eqnarray}
\|LL^T\|_F^2&=&\|QX^T\|_F^2+\|Q_{\perp}X_{\perp}^T\|_F^2\\
&= & \|QX^T\|_F^2+Tr(Q_{\perp}X_{\perp}^TX_{\perp}Q_{\perp}^T)\\
& = & \|QX^T\|_F^2+Tr(Q_{\perp}Q_{\perp}^T)\\
& = & \|QX^T\|_F^2+\|Q_{\perp}\|_F^2
\end{eqnarray} where $\|\cdot\|_F$ is the Frobenius norm. Similarly, the bias $LX-I=QX^TX-I$ does not depend on $Q_{\perp}$. Thus, any non-zero value of $Q_{\perp}$ will strictly increase the value of the objective relative to setting $Q_{\perp}=0$, without having any effect on the constraint.
\end{proof}

\begin{corollary}
The optimal solution to Equation \ref{main_objective} takes the form $L=(I+Q)G^{-1}X^T$ where 
\begin{equation}
\label{Q_objective}
Q=argmin_{Q'\in\mathbb{R}^{d\times d}}Tr(Q'G^{-1}Q'^T)/2+Tr(Q'G^{-1})+\alpha^{-1}\|Q'\|_p
\end{equation}
and $\alpha\geq 0$ is determined by $C$.
\end{corollary}
\begin{proof}
By the Lemma, it is no loss of generality to assume that $L$ has the indicated form. Now plug in to Equation \ref{main_objective} and simplify. The $\alpha$ term is just converting the constraint to a Lagrange multiplier.
\end{proof}

By the discussion above, we conclude that there is \textit{exactly one} minimum of Equation \ref{Q_objective}. 

\begin{proposition}
\label{prop_commutes}
If $Q$ is the solution to the equation \ref{Q_objective}, then $Q$ is symmetric and commutes with $G$
\end{proposition}
Before giving the proof, we first present a more technical matrix lemma. 
\begin{lemma}
\label{lemma_diag}
For any square matrix $M$ and $p\geq 1$, $\|M_d\|_p\leq \|M\|_p$, where $M_d$ is the diagonal part (i.e., the matrix with all non-diagonal entries set to zero).
\end{lemma}
\begin{proof}
Evidently the singular values of $M_d$ coincide with the absolute values of the diagonal entries. By an inequality of Ky Fan \citep{kyfan}, 
\begin{equation}\sum_{i\leq k} \sigma_i(M_d)\leq \sum_{i\leq k} \sigma_i(M)\end{equation}
for any $1\leq k\leq d$, where $\sigma_i$ denotes the singular values, ordered from largest to smallest. The lemma now follows from Schur convexity of the Euclidean p-norm.
\end{proof}

\begin{proof} (of proposition \ref{prop_commutes})
 Let $G^{-1}=UDU^T$ be the diagonalization, where $U$ is an orthogonal matrix and $D$ is diagonal. The objective is 
 \begin{eqnarray}
Q&=&argmin_QTr(Q'UDU^TQ'^T)/2+Tr(Q'UDU^T)+\alpha^{-1}\|Q'\|_p\\
& = &argmin_QTr(U^TQ'UDU^TQ'^TU)/2+Tr(U^TQ'UD)+\alpha^{-1}\|U^TQ'U\|_p\\
UQU^T& = &argmin_PTr(PDP^T)/2+Tr(PD)+\alpha^{-1}\|P\|_p\\
 \end{eqnarray}
 where we reparametrized as $P:=U^TQ'U$. Evidently, the proposition will follow if we can show that the minimal $P$ is diagonal. Let $P=P_d+P_{od}$ be the decomposition into diagonal and off-diagonal parts\footnote{By definition an \textit{off-diagonal matrix} is one with all zeros along the diagonal.}. Plugging into the objective
 \begin{eqnarray}
Obj(P) &=&Tr((P_d+P_{od})D(P_d+P_{od}^T))/2+Tr(P_dD)+Tr(P_{od}D)+\alpha^{-1}\|P\|_p\\
 &=& Tr(P_dDP_d)/2+Tr(P_dD)+Tr(P_{od}DP_{od}^T)/2+Tr(P_{od}D)+\alpha^{-1}\|P\|_p\\
 &+&Tr(P_dDP_{od}^T)
 \end{eqnarray}
Now, we note that in general if $A$ is diagonal, and $B$ is off-diagonal,then $Tr(AB)=0$. So the expression simplifies to 
\begin{equation}Tr(P_dDP_d)/2+Tr(P_dD)+Tr(P_{od}DP_{od}^T)/2+\alpha^{-1}\|P\|_p\end{equation}
Now, the third term is non-negative because it is the trace of a PSD matrix; thus $Obj(P)\geq Tr(P_dDP_d)/2+Tr(P_dD)+\alpha^{-1}\|P\|_p$. By Lemma \ref{lemma_diag}, $\|P\|_p\geq \|P_d\|_p$, and therefore 
\begin{equation}Obj(P)\geq Obj(P_d)\end{equation}
However, $P$ was assumed to be the minimum, which implies by strict convexity that $P=P_d$
\end{proof}

By the above, we now know that the optimal $Q$ must satisfy $Q=U diag(q) U^T$, where $q$ denotes the vector of eigenvalues and $U$ is the matrix of eigenvectors of $G^{-1}$. Since $\hat{G}^{-1}=(1+Q)G^{-1}$ in the notation of the theorem statement, we have shown the first two claims, namely that $\hat{G}$ is symmetric and commutes with $G$. 

By plugging into the objective in \ref{Q_objective} and simplifying, we see that 

\begin{equation}
\label{q_objective_diag}
q=argmin_{x\in\mathbb{R}^d} {\frac 1 2}\sum_i x_i^2/\sigma_i^2+\sum_i x_i/\sigma_i^2 +\alpha^{-1}|x|_p
\end{equation}
where $\sigma_i^2$ are the eigenvalues of $G$, and $|\cdot|_p$ denotes the Euclidean p norm of a vector. 

To show the last claim that $\hat{G}-G$ is non-negative definite, it is enough to show that $0\geq q_i\geq -1$,since $\lambda_i(\hat{G})=\lambda_i(G)/(1+q_i)$.It is easy to see that the eigenvectors of $Q$ must be non-positive (since if any eigenvector is positive, then switching the sign will leave the first and third terms alone, while strictly decreasing the second term). To see the second inequality, we first add a suitable constant to the objective and rewrite it as ${\frac 1 2}\sum_i (x_i+1)^2/\sigma_i^2+\alpha^{-1}|x|_p$. Supposing that $q_i<-1$ for some $i$, let us replace it with $q_i'=-2-q_i$. Doing so does not change the first term (i.e. $(q_i+1)^2=(q_i'+1)^2$), however it strictly decreases the second term, since $|q_i|'<|q_i|$ \footnote{e.g., square both sides}, contradicting the minimality of $q$\footnote{This argument doesn't quite go through if $p=\infty$, since in that case decreasing the magnitude of a component of $q$ might not decrease the norm. But that's fine because we will derive the exact solution for $p=\infty$ shortly.}

Until now, we have not made any assumption about $p$ except that $p\geq 1$. At this point, we separately analyze each of the three special cases $p=1,2,\infty$. 

Before doing so, however, we first present the following well-known and elementary calculation, which we will employ several times in what follows:

\begin{proposition}
\label{prop_abs_min}
Let $y,\tau>0$ then 
\begin{equation}argmin_{x\in\mathbb{R}}{\frac 1 2}(x+y)^2+\tau |x|=min(\tau-y,0)\end{equation}
\end{proposition}
\begin{proof}
It is clear that the optimal $x$ cannot be positive, so we want to compute 
\begin{equation}argmin_{x\leq 0}{\frac 1 2}(x+y)^2-\tau x=argmin_{x\leq 0}{\frac 1 2}(x+y-\tau)^2+{\frac {y^2-(y-\tau)^2}2}\end{equation}
where we completed the square. The second term does not depend on $x$ and therefore has no effect on the argmin. Now the formula is immediate. If $\tau-y<0$, then the minimum is plainly attained at $x=\tau-y<0$. If $\tau-y>0$, then the parabola is monotonically decreasing on the interval $(-\infty,0)$, so the minimum is attained at $x=0$. 

\end{proof}

\textbf{Nuclear case (p=1)}

The objective \ref{q_objective_diag} splits into a sum of separable one-dimensional problems: 
\begin{equation}argmin_{x\in\mathbb{R}} {\frac 1 2}x^2/\sigma_i^2+x/\sigma_i^2+\alpha^{-1}|x|\end{equation}
\begin{equation}argmin_{x\in\mathbb{R}} {\frac 1 {2}}(x+1)^2+\sigma_i^2\alpha^{-1}|x|\end{equation}
By Proposition\ref{prop_abs_min},
\begin{equation}q_i=min(\sigma_i^2/\alpha-1,0)\end{equation}
The eigenvalues of $\hat{G}$ are thus given by 
\begin{equation}\lambda_i(\hat{G})={\frac {\sigma_i^2}{1+q_i}}={\frac {\sigma_i^2}{min(\sigma_i^2/\alpha,1)}}=max(\sigma_i^2,\alpha)\end{equation}

as claimed.

\textbf{Frobenius case (p=2)}

This is a simple calculus exercise and omitted. 

\textbf{Spectral case (p=$\infty$)}

W know that there exists exactly one minimizer $q_{opt}\in\mathbb{R}^d$ of \ref{q_objective_diag}. Fix some $C>max(|q_{opt}|_{\infty},max_i(1+\sigma_i^2/\alpha))$ and consider the problem
\begin{equation}
\label{l_infinity_obj}
\min_{q: |q|_{\infty}\leq C} {\frac 1 2}\sum_i q_i^2/\sigma_i^2+\sum_i q_i/\sigma_i^2 +\alpha^{-1}|q|_{\infty}
\end{equation}
This clearly has the same minimum as the original unconstrained problem.

Note the following identity: 
\begin{equation}\alpha^{-1}|q|_{\infty}=max_{y: |y|_1\leq \alpha^{-1}}\langle q,y\rangle\end{equation}

By Von Neumann's minimax theorem, we can interchange the min and the max. So the optimal value of the objective is:

\begin{equation}
\label{eqn_minmax}
max_{y: |y|_1\leq \alpha^{-1}}min_{q:|q|_{\infty}\leq C} {\frac 1 2}\sum_i q_i^2/\sigma_i^2+\sum_i q_i/\sigma_i^2 +\langle q,y\rangle 
\end{equation}

Now, the set $\{q:|q|_{\infty}\leq C\}$ is geometrically a product of intervals $[-C,C]^d$. Thus, we see that the inner objective splits into a sum of uncoupled 1-dimensional problems:
\begin{equation}q_i=argmin_{q_i\in [-C,C]} {\frac 1 2}q_i^2/\sigma_i^2+q_i/\sigma_i^2+q_iy_i=argmin_{q_i\in[-C,C]}{\frac 1 {2\sigma_i^2}}(q_i+1+\sigma_i^2y_i)^2-{\frac 1 {2\sigma_i^2}}(1+\sigma_i^2y_i)^2\end{equation}

By assumption on $C$,  $1+\sigma_i^2y_i\leq C$, and therefore for any fixed $y_i$, the minimum of the inner problem is attained at 
\begin{equation}
\label{eqn_xy}
q_i=-1-\sigma_i^2y_i
\end{equation}, with minimal value equal to $-{\frac 1 {2\sigma_i^2}}(1+\sigma_i^2y_i)^2=-{\frac {\sigma_i^2} {2}}(\sigma_i^{-2}+y_i)^2$. Plugging back in to \ref{eqn_minmax}, we need to solve 
\begin{equation}min_{y:|y|_1\leq \alpha^{-1}}{\frac 1 2}\sum_i \sigma_i^2 (\sigma_i^{-2}+y_i)^2\end{equation}

For this, we introduce a Lagrange multiplier $\lambda$, upon which the objective splits again into a sum of uncoupled 1d problems:
$y_i=argmin_{y\in\mathbb{R}} {\frac 1 2}\sigma_i^2 (\sigma_i^{-2}+y)^2+\lambda |y|$.

Using Proposition \ref{prop_abs_min}, we derive the solution 
\begin{equation}y_i=\sigma_i^{-2}min(\lambda -1,0)=\sigma_i^{-2}(min(\lambda,1)-1)\end{equation}
where $\lambda$ is chosen to satisfy the original constraint $\sum_i |y_i|\leq \alpha^{-1}$. Using the relation \ref{eqn_xy} between $q_i$ and $y_i$
\begin{equation}q_i=-1-(min(\lambda,1)-1)=-min(\lambda,1)\end{equation}

for the solution to the problem \ref{l_infinity_obj}. Since we took $C$ to be large enough to contain the solution to the unconstrained problem, we conclude that this is also the solution to the unconstrained problem (i.e. $C=\infty$). Since $x_i$ are the eigenvalues of $Q$, we conclude that the eigenvalues of $\hat{G}$ are 
\begin{equation}\lambda_i(\hat{G})={\frac {\sigma_i^2}{1+q_i}}={\frac {\sigma_i^2}{1-min(\lambda,1)}}\end{equation}
Clearly the denominator lies in $[0,1]$, therefore we recover the claimed form in which all eigenvalues of $\hat{G}$ are obtained by scalar multiplication with some factor $>1$. Note that the case $\lambda>1$ corresponds to multiplication by infinity, i.e. setting $\hat{G}^{-1}$ (and thus $\hat{\beta})$ to zero.

\subsection{Proof of \ref{nuclear_gen_error}}
\label{nuclear_gen_error_proof}
We first give the definition of the Appel hypergeometric function $F_1$ for reference. 

The function is typically defined as
\begin{equation}F_1(\alpha,\beta,\beta',\gamma,x,y)	=	\sum_{m=0}^{\infty}\sum_{n=0}^{\infty}{\frac {(\alpha)_{m+n}(\beta)_{m}(\beta')_n}{m!n!(\gamma)_{m+n}}}x^my^n\end{equation}
Here $(\cdot)_m$ is the Pochammer symbol. The series is absolutely convergent for $|x|,|y|<1$, and arbitrary $\alpha,\beta,\beta',\gamma$. In \ref{nuclear_gen_error}
we may possibly need to evaluate it at some $x$ outside of this range; we do this by appropriate analytic continuation, discussed further below.  

To prove the formula, we note that the extremal cases follow straightforwardly from the integral formula. So in what follows we will assume that $\alpha\in[\lambda_-,\lambda_+]$. 

We rewrite the integral
\begin{eqnarray}
{\frac {Err}{\lambda}}-\beta^2 &=& \int_{\lambda_-}^{\lambda_+}\beta^2x^2f_{\alpha}(x)^{-2}-2\beta^2 xf_{\alpha}^{-1}(x)+\sigma^2 xf^{-2}_{\alpha}(x)d\mu\\
& = & \int_{\lambda_-}^{\alpha}\beta^2\alpha^{-2}x^2-2\beta^2\alpha^{-1}x+\sigma^2\alpha^{-2} xd\mu\\
&+&\int_{\alpha}^{\lambda_+}\beta^2-2\beta^2+\sigma^2 x^{-1}d\mu\\
& = & \beta^2\alpha^{-2} \int_{\lambda_-}^{\alpha} x^2d\mu +(\sigma^2\alpha^{-2}-2\beta^2\alpha^{-1})\int_{\lambda_-}^{\alpha} xd\mu\\
&-&\beta^2 (1-F_{\lambda}(\alpha))+ \sigma^2({\frac 1 {1-\lambda}}-\int_{\lambda_-}^\alpha x^{-1}d\mu)\\
& = & c +{\frac {\beta^2}{\alpha^{2}}}I(2,\alpha)+(\sigma^2\alpha^{-2}-2\beta^2\alpha^{-1})I(1,\alpha)+\beta^2F_{\lambda}(\alpha)-\sigma^2 I(-1,\alpha)
\end{eqnarray}
where we defined $I(r,\alpha)=\int_{\lambda_-}^{\alpha}x^rd\mu$. 

So we have reduced the theorem to just evaluating $I(r,\alpha)$ for $r\in \{-1,1,2\}$. Now, we plug in the form of the MP density, and make the variable substitution $u={\frac {x-\lambda_-}{\alpha-\lambda_-}}$. 
Clearly then $u(\alpha-\lambda_-)+\lambda_-=x$.

\begin{eqnarray*}
2\pi I(r,\alpha)& = & \int_{\lambda_-}^{\alpha} x^{r-1}\sqrt{(\lambda_+-x)(x-\lambda_-)}dx\\
&=& \int_0^1 (u(\alpha-\lambda_-)+\lambda_-)^{r-1}\sqrt{\lambda_+-\lambda_--u(\alpha-\lambda_-)}\sqrt{u(\alpha-\lambda_-)}(\alpha-\lambda_-)du\\
&=&(\alpha-\lambda_-)^{3/2}\lambda_-^{r-1}\sqrt{\lambda_+-\lambda_-}\int_0^1 \sqrt{u}(1+u{\frac {\alpha-\lambda_-}{\lambda_-}})^{r-1}\sqrt{1-u{\frac {\alpha-\lambda_-}{\lambda_+-\lambda_-}}}du
\end{eqnarray*}

We can now express this in the form given in the proposition by means of the following formula \citep{bailey}: 
\begin{equation}F_1(\alpha,\beta,\beta',\gamma,x,y)={\frac {\Gamma(\gamma)}{\Gamma(\alpha)\Gamma(\gamma-\alpha)}}\int_0^1u^{\alpha-1}(1-u)^{\gamma-\alpha-1}(1-ux)^{-\beta}(1-uy)^{-\beta'}du\end{equation}
if $\alpha,\gamma-\alpha>0$. We use this formula to define the analytic continuation in case $|x|>1$ as can happen in formula \ref{nuclear_gen_error}.
\qed

\subsection{Proof of Generalization Formulas \ref{prop_gaussian_mse} and \ref{prop_mse_diag}}
\label{mse_proof}
We first consider \ref{prop_mse_diag}. Given the setup in the main text, consider computing the test error for a fixed training/testing pair:
\begin{equation}|X_{test} (\hat{G}^{-1}X_{tr}^T)(X_{tr}\beta_0+\epsilon)-X_{test}\beta_0|^2=|X_{test} (\hat{G}^{-1}G -I)\beta_0+X_{test} \hat{G}^{-1}X_{tr}^T\epsilon|^2\end{equation}
\begin{equation}=|X_2DS B\beta_0+X_2D S\hat{G}^{-1}X_{tr}^T\epsilon|^2\end{equation}
where $S$ is the diagonal matrix containing the noise values $\sqrt{s_i}$, and $D$ is the diagonal matrix containing $\sqrt{\lambda_i}$. 
By rotational symmetry of the Frobenius norm, when we marginalize $X_2$ we get 
\begin{equation} |DSB\beta_0+DS\hat{G}^{-1}X_{tr}^T\epsilon|^2\end{equation}
And marginalizing over $\epsilon$ further yields 
\begin{equation}|DSB\beta_0|^2+\sigma^2 Tr (DS\hat{G}^{-1}G\hat{G}^{-1}SD)\end{equation}
where the cross term vanishes because $\mathbb{E}\epsilon=0$.
Remembering all matrices are actually diagonal, we obtain 
\begin{equation}n*err=\sum_{i=1}^d s_i\lambda_i(\beta_0)_i^2 (\lambda_i/f_{\alpha}(\lambda_i)-1)^2+\sigma^2\sum_i s_i\lambda_i (\lambda_i/f_{\alpha}(\lambda_i)^2)\end{equation}
Since $s_i$ are independent of $\lambda_i$ and have expectation 1, we can easily marginalize out, obtaining:
\begin{equation}\sum_{i=1}^d\lambda_i(\beta_0)_i^2 (\lambda_i/f_{\alpha}(\lambda_i)-1)^2+\sigma^2\sum_i \lambda_i (\lambda_i/f_{\alpha}(\lambda_i)^2)\end{equation}

At this point, the only variable that remains to marginalize over is $\lambda$. The result is:

\begin{equation}d^{-1}|\beta_0|^2 \sum_i  E_{\lambda_1,\dots, \lambda_n}\lambda_i\left( (\lambda_i/f_{\alpha}(\lambda_i)-1)\right)^2+\sigma^2 E_{\lambda_1,\dots, \lambda_n}\sum_i \lambda_i^2/f_{\alpha}(\lambda_i)^2 \end{equation}
where we used exchangeability of $\lambda_i$ to bring out the factor of $|\beta_0|^2$. Now the proposition follows by taking the limit $n\to\infty$ and applying the law of large numbers. 
\qed 

The proof of \ref{prop_gaussian_mse} is very similar - as above, the idea is to use rotational symmetry to reduce the squared-error to a sum of the form $E_{\lambda_1,\dots,\lambda_d}\sum_i g(\lambda_i)$. The main difference is that in this case the eigenvalues of $G$ are no longer independent, but they are at least still exchangeable, so we can bring out the factor of $|\beta_0|^2$ in front of the sum like above. And since $G$ now has a Wishart distribution, we can use the Marchenko-Pastur theorem \citep{bai} instead of the Law of large numbers to reduce the sum to the indicated integral form.
\subsection{Estimation of minima and curvatures}
\label{min_estimation}
To estimate the minimum and curvature of an error curve $Err(\alpha)$, we evaluate $Err(\alpha_i), i=1,\dots, n$, where where $n=500$ and $\alpha_i$ are equally logarithmically spaced between $10^{-3}$ and $10^5$. The minimum is simply estimated as $\min_i Err(\alpha_i)$. To estimate the curvature at the minimum, we took the closest few points to the minimum and fit a quadratic function. That is, if $i_0=argmin_i Err(\alpha_i)$ then we fit a two-parameter linear model of the form
\begin{equation}Err(\alpha_i)-Err(\alpha_{i_0})\sim a(\alpha_i-\alpha_{i_0})+b(\alpha_i-\alpha_{i_0})^2, i_0-5\leq i\leq i_0+5\end{equation}
with the estimated Hessian being $2b$.
\subsection{Relation between $\alpha$ and $C$ in Theorem \ref{main_theorem}}
\label{alpha_vs_C}
To express the relation between $\alpha$ and $C$, we consider two cases, corresponding to whether or not the constraint set contains the global minimizer of the objective $L\mapsto Tr(LL^T)$.
The first case is when $C\geq \|I_d\|_p=d^{1/p}$. In this case, the optimum is evidently $L=0_{d\times N}$, corresponding to $\alpha=\infty$. 

In the case where $C<d^{1/p}$, $0_{d\times N}$ does not lie in the constraint set. Therefore, the optimum $L$ lies on the boundary of the constraint set, i.e. $\|LX-I\|_p=C$. By plugging in the functional forms from Theorem \ref{main_theorem}, we obtain the following relations:
\begin{eqnarray}
\sum_{\sigma_i<\alpha} 1-{\frac {\sigma_i}{\alpha}} &=&C\hspace{1cm}(p=1)\\
\sqrt{\sum_i {\frac {\alpha^2}{(\sigma_i+\alpha)^2}}}&=&C\hspace{1cm} (p=2)\\
\alpha/(1+\alpha)&=&C\hspace{1cm} (p=\infty)\\
\end{eqnarray}
where $\sigma_i$ are the eigenvalues of $X^TX$.

\subsection{Effect of number of $\alpha$ values}
As pointed out in the main text, the performance of the cross-validated models can depend on the number $n$ of $\alpha$ values used for cross-validation. To do so, we follow the methodology used in \ref{gaussian_predictor}, with the only difference being we use a smaller hyperparameter range $\sigma\in [.5,2,3.5],\lambda=.5,\rho\in [0,.5,.9]$, and also vary the total number $n$ of $\alpha$ values used in the cross validation $n\in [9,15,20,30,50]$. We keep the limits of the range of $\alpha$ values as before, and also maintain the equal logarithmic spacing.  

We show the average error and win probability in Figure \ref{n_alpha_mse}
 and Figure \ref{n_alpha_winprob} respectively. As per the discussion in Sections \ref{opt_reg_form_section} and \ref{discussion_section}, we see that the Ridge often benefits drastically from increased number of $\alpha$s, and can overtake the Nuclear for large number of $\alpha$ in cases when the Nuclear performs better for small number of $\alpha$. 
 \begin{figure}[h]
\begin{center}
\includegraphics[width=14cm,height=12cm]{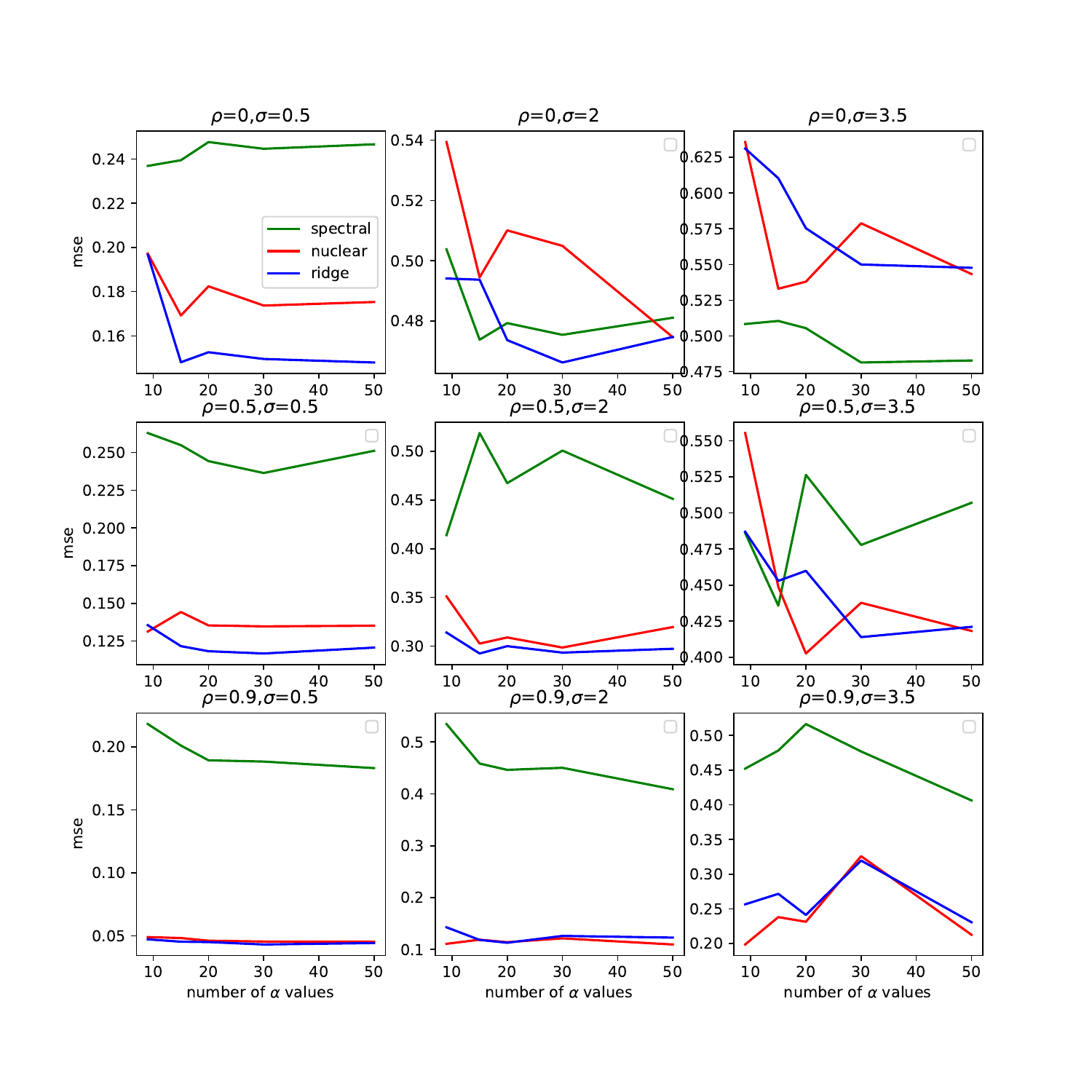}

\end{center}
\caption{Average test error for varied number of regularization strength values $\alpha$ used in cross-validation. Each point represents an aggregate over 100 datasets. }
\label{n_alpha_mse}
\end{figure}

\begin{figure}[h]
\begin{center}
\includegraphics[width=14cm,height=12cm]{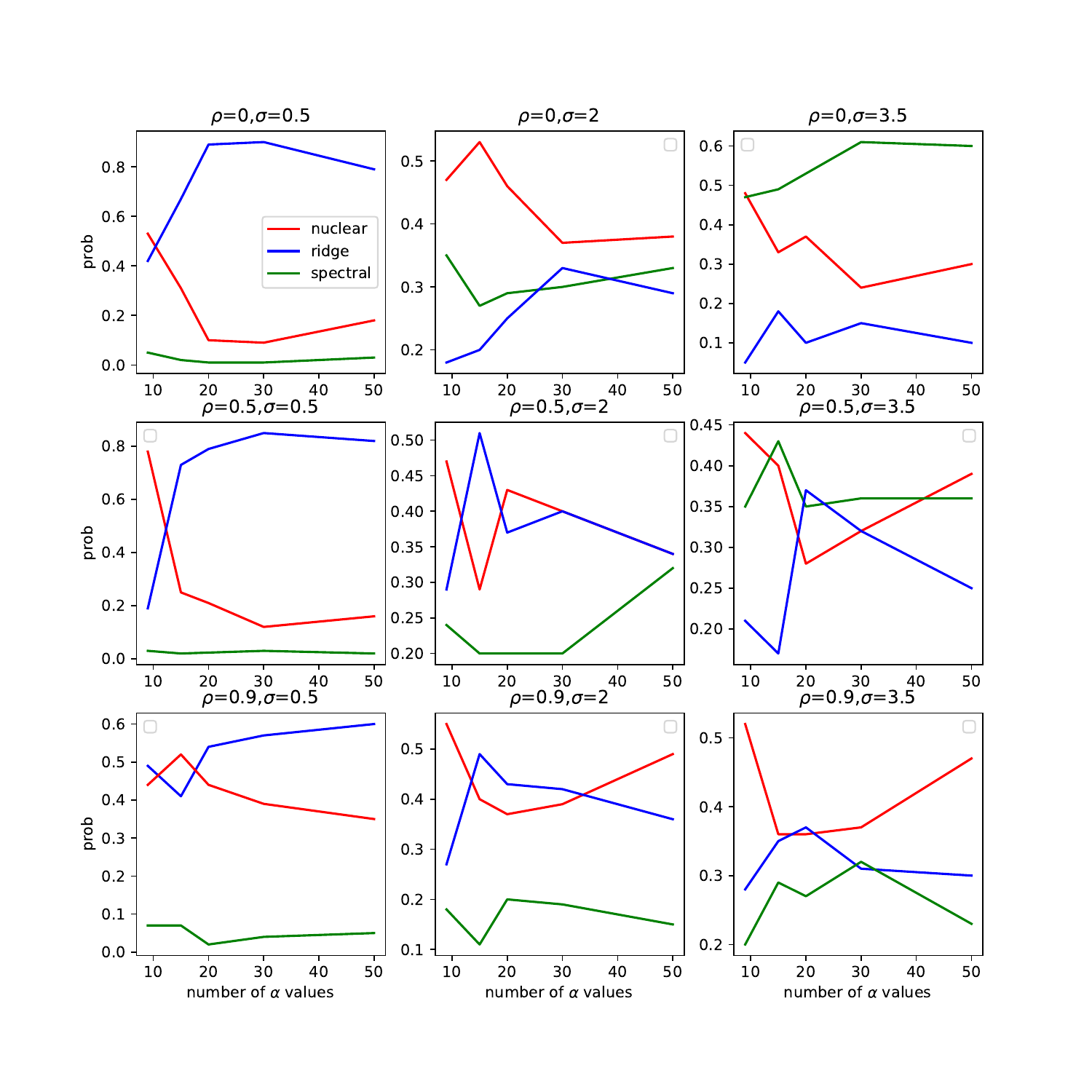}

\end{center}
\caption{Same as Figure \ref{n_alpha_mse}, except showing the probability that each model attains the lowest test error on a given dataset. Each point represents an aggregate over 100 datasets.}
\label{n_alpha_winprob}
\end{figure}

\subsection{Sparsely structured data and comparison to Lasso}
We consider a variant of the setup in Section \ref{gaussian_predictor}, in which the data is constructed to have sparse structure. In this case, when generating the ground truth coefficient vector $\beta_0$, we select a set of indices $I\subset \{1,\dots, 10\}$, $|I|=3$ at random, and generate $\beta_0$ as 
\begin{eqnarray}
(\beta_0)_i&=&N(0,1), i\in I\\
(\beta_0)_i&=&N(0,1)/10, i\not\in I;
\end{eqnarray}
We also use the smaller hyperparameter ranges
$\rho=0,\lambda=.5,\sigma\in [.5,1,1.5,2,2.5,3]$. Otherwise we follow the methodology of Section \ref{gaussian_predictor}.

We also include Lasso in the set of considered models, since it is designed to deal with sparse coefficient vectors. Note, however, that Lasso is not manifestly a Linear Estimator \footnote{We do not, however, have a proof that it cannot be expressed in the requisite form of Equation \ref{linear_estimator_defn}} in the sense defined in Section \ref{main_theorem_section} .

We show the average error and win probability in Figure \ref{sparse_mse}
 and Figure \ref{sparse_winprob} respectively. 

\begin{figure}[h]
\begin{center}
\includegraphics[width=14cm,height=8cm]{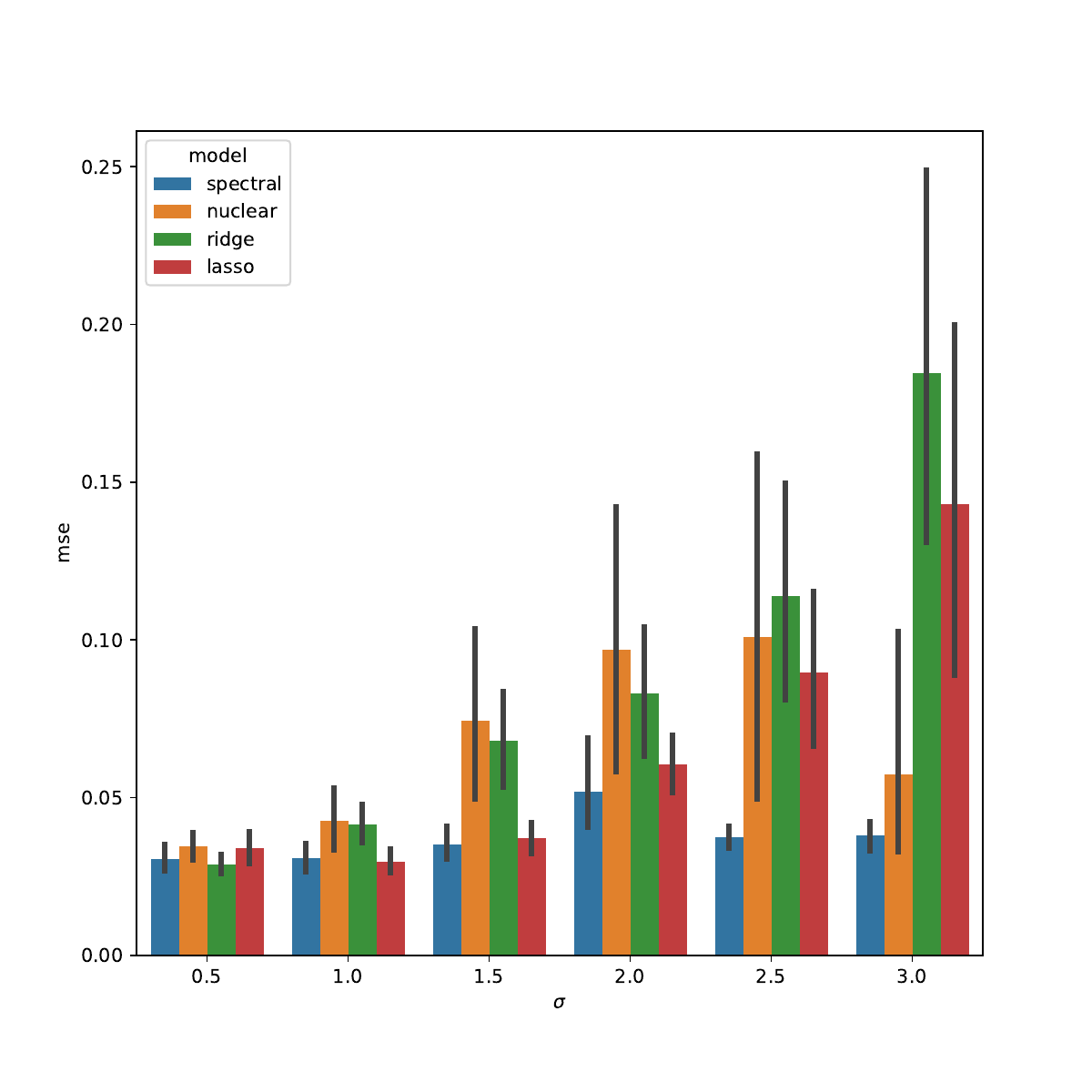}

\end{center}
\caption{Average test error, with 95 percent confidence intervals, on Gaussian data with sparse ground truth coefficient vector. Each bar is an aggregate over 100 datasets.}
\label{sparse_mse}
\end{figure}

\begin{figure}[h]
\begin{center}
\includegraphics[width=14cm,height=10cm]{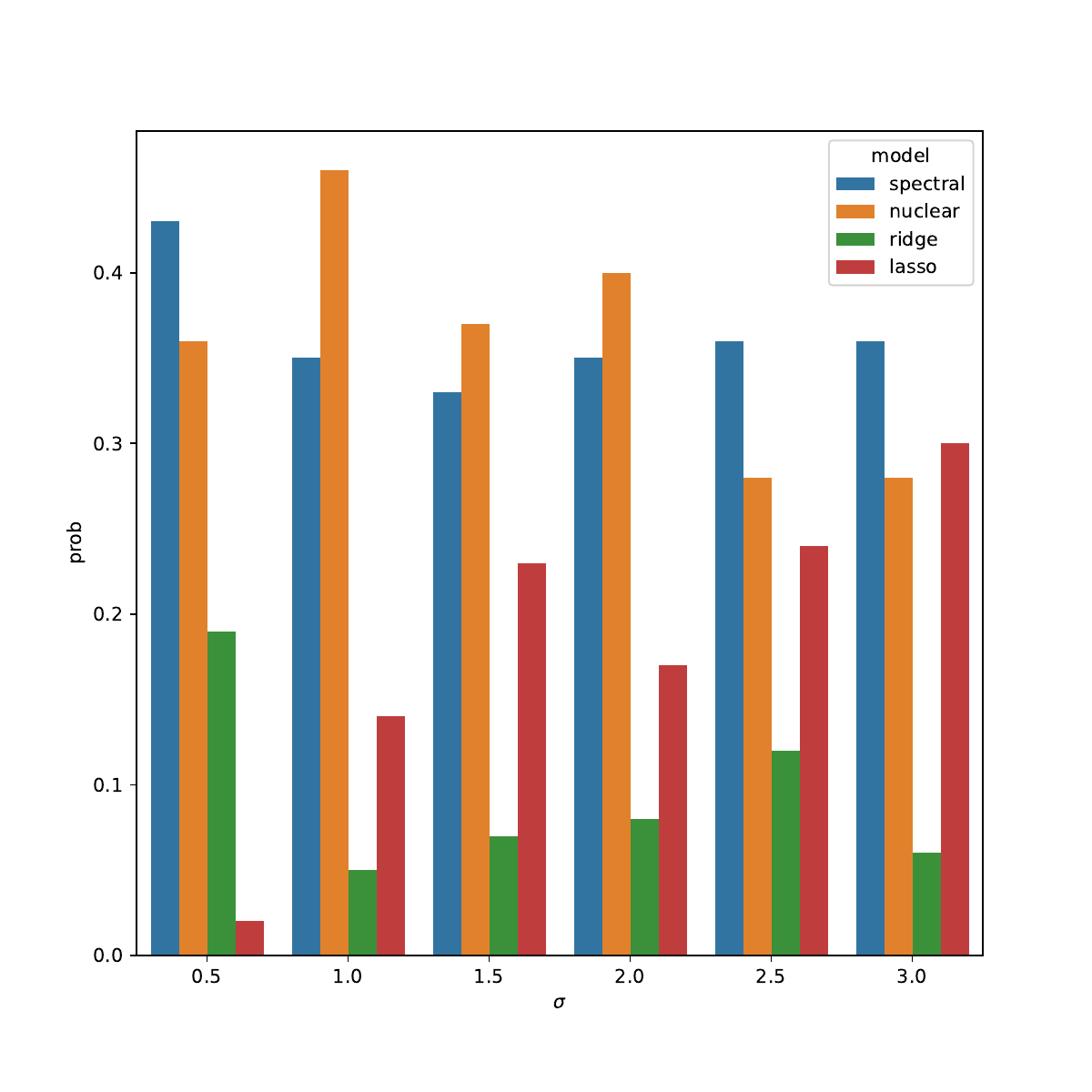}

\end{center}
\caption{Same as Figure \ref{sparse_mse}, except showing the probability that each model attains the lowest test error on a given dataset. Each bar is an agggregate over 100 datasets.}
\label{sparse_winprob}
\end{figure}

\subsection{Real data experiments}
We evaluate the models on real (i.e., non-synthetic) data. We consider the well-known Diabetes dataset (N=442, d=10) and California housing dataset (N=20640,d=8), both available from the \verb|sklearn.datasets| library. 

To analyze the performance of the models on each dataset, we create random train-test splits in which the size of the training set is always set to 300, and the test set comprises the remaining observations. Each model is fit on the training set (including the regularization strength $\alpha$, using the same cross-validation procedure as in Section \ref{gaussian_predictor}), and the mean-square-error is evaluated on the test set. We use the same set of 9 $\alpha$ values as in Section \ref{gaussian_predictor} for all models. We constructed a total of 200 splits in this way , and evaluated each model on each split as before (so that we can evaluate each model's probability of winning on a given split, as well as the average error over splits). 

We show the average error and win probability in Figure \ref{real_data_mse}
 and Figure \ref{real_data_winprob} respectively. We see a similar pattern as in Figure \ref{fig_cv}, with the Nuclear having a clear advantage over the other models in terms of win probability. The Nuclear also attains the lowest average MSE in the diabetes data, and is essentially tied with Ridge for lowest average MSE in the housing data. 

\begin{figure}[h]
\begin{center}
\includegraphics[width=14cm,height=8cm]{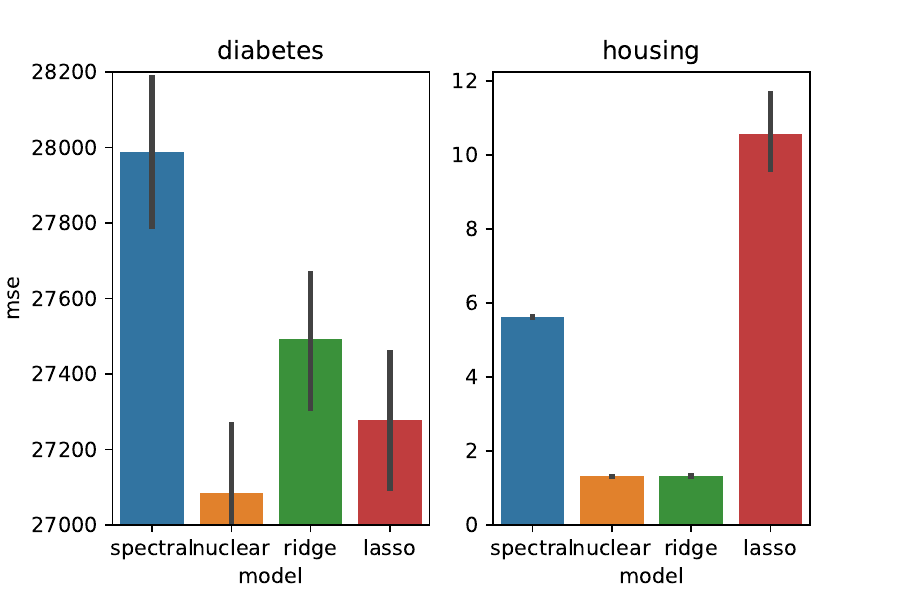}
\end{center}
\caption{Average test error, with 95 percent confidence intervals, on a random train-test split, for the diabetes and California housing datasets. Each bar is an aggregate over 200 splits. }
\label{real_data_mse}
\end{figure}

\begin{figure}[h]
\begin{center}
\includegraphics[width=14cm,height=8cm]{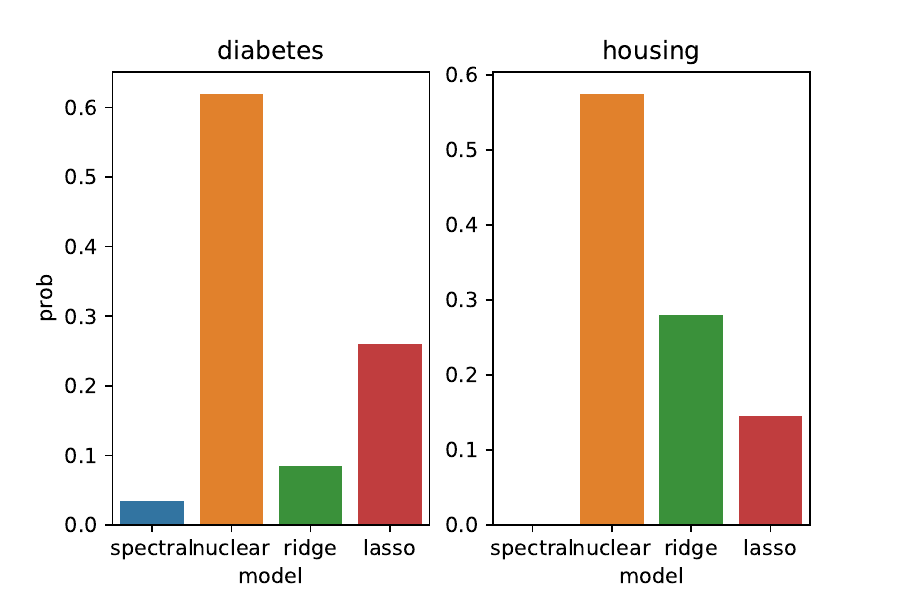}

\end{center}
\caption{Same as Figure \ref{real_data_mse}, except showing the probability that each model attains the lowest test error on a given split. Each bar is an aggregate over 200 splits.}
\label{real_data_winprob}
\end{figure}


\end{document}